\documentclass{article}
\usepackage{kr}

\usepackage{times}
\usepackage{soul}
\usepackage{url}
\usepackage[hidelinks]{hyperref}
\usepackage[utf8]{inputenc}
\usepackage[small]{caption}
\usepackage{graphicx}
\usepackage{amsmath}
\usepackage{natbib}
\usepackage{amsthm}
\usepackage{booktabs}
\usepackage{algorithm}
\usepackage{algorithmic}
\usepackage{boxedminipage}
\let\cite\citep

\newtheorem{example}{Example}
\newtheorem{theorem}{Theorem}
\newtheorem{proposition}{Proposition}
\newtheorem{lemma}{Lemma}

\newtheorem{corollary}{Corollary}
\newtheorem{definition}{Definition}

\usepackage{thmtools}
\usepackage{thm-restate}
\usepackage{amsfonts}
\usepackage{xspace}
\usepackage{pifont}
\usepackage{color}
\usepackage{upgreek, stmaryrd}
\usepackage{amssymb}
\usepackage{misc}   
\usepackage{xargs}
\usepackage[pdftex,dvipsnames]{xcolor}
\usepackage{mathtools}
\usepackage{enumitem}
\usepackage{tikz}
\usetikzlibrary{arrows.meta, positioning}

\usepackage{microtype}

\newcommand{\bound}[1]{\ensuremath{2^{\| {#1}\|^2}}}

\title{Fitting Description Logic Ontologies to ABox and Query Examples}

\author{%
Maurice Funk\and
Marvin Grosser\and
Carsten Lutz
\affiliations
Leipzig University\\
ScaDS.AI Center Dresden/Leipzig\\
\emails
\{mfunk, grosser, clu\}@informatik.uni-leipzig.com
}

\begin{document}

\maketitle

\begin{abstract}
  We study a fitting problem inspired by ontology-mediated querying: given a collection
  of positive and negative examples of
  the form $(\Amc,q)$ with 
  \Amc an ABox and $q$ a Boolean query, we seek
  an ontology \Omc that satisfies $\Amc \cup \Omc \models q$ for all positive examples and $\Amc \cup \Omc \not\models q$ for all negative examples.
  We consider the description logics \ALC and \ALCI as ontology languages and 
  a range of query languages that
  includes atomic queries (AQs), conjunctive queries (CQs), and unions thereof (UCQs). 
  For all of the resulting fitting problems,
  we provide
  effective characterizations and determine the computational complexity
  of deciding whether a fitting ontology exists. This problem turns out to be \coNPclass-complete for AQs and full CQs
  and \TwoExpTime-complete for CQs and UCQs.
  These results hold for both \ALC and \ALCI.
\end{abstract}

\section{Introduction}

In many areas of computer science and AI, a fundamental problem is to fit a formal object to a given collection of  examples. In inductive program synthesis, for instance, one wants to find a program that complies with a given collection of examples 
of input-output behavior \cite{ProgramSynthesisReview}. In machine learning, fitting a model to a given set of examples
is closely linked to PAC-style generalization guarantees \cite{DBLP:books/daglib/0033642}. And in database
research, the traditional query-by-example paradigm asks to find a query that fits a given set of data examples \cite{Li2015:qfe}.

In this article, we study the problem of fitting an ontology formulated in a description logic (DL) to a given collection of positive and negative 
examples. Our concrete setting is motivated
by the paradigm of ontology-mediated querying
where data is enriched by an ontology that
provides domain knowledge, aiming to return
more complete answers and to bridge heterogeneous representations in the data \cite{DBLP:conf/rweb/BienvenuO15,DBLP:conf/ijcai/XiaoCKLPRZ18}.
Guided by this application, we use examples that take the form $(\Amc,q)$ where \Amc is an ABox
(in other words: a database) and $q$ is a Boolean query. We then seek an ontology \Omc that
satisfies $\Amc \cup \Omc \models q$ for
all positive examples  and $\Amc \cup \Omc \not\models q$ for
all negative examples. It is not a restriction
that $q$ is required to be Boolean since our queries may 
contain individuals from the ABox. 

A main application of this ontology fitting problem is to assist with ontology construction and engineering. This is in 
the spirit of several other proposals that
have the same aim, such as ontology construction and completion using formal
concept analysis \cite{DBLP:conf/ijcai/BaaderGSS07,DBLP:conf/icfca/BaaderD09,DBLP:conf/aaai/Kriegel24} and Angluin's framework of exact
learning \cite{DBLP:journals/jmlr/KonevLOW17}, see also the 
survey \cite{DBLP:journals/ki/Ozaki20}. We remark that
there is a large literature on fitting 
DL concepts (rather than ontologies) to a collection of
examples, sometimes referred to as concept
learning, see for instance \cite{DBLP:journals/ml/LehmannH10,DBLP:conf/www/BuhmannLWB18,DBLP:conf/ijcai/FunkJLPW19,DBLP:conf/kr/JungLPW21}. Concepts can be viewed as the building
blocks of an ontology and in fact concept fitting also has the support of ontology engineering as a main aim. The techniques needed for
concept fitting and ontology fitting are, however, quite different. 
While it is probably unrealistic to assume that an ontology for an entire domain can be built in a single step from a given set of examples, we believe that small portions of the ontology can be constructed this way, thereby supporting a step-by-step development process by a human engineer. Moreover, in ontology-mediated querying there are applications where a more pragmatic
view of an ontology seems appropriate:
instead of providing a careful and detailed domain representation, one only wants the 
ontology to support more complete answers for some given query or a small set of
queries \cite{DBLP:journals/is/CalvaneseDNRT06,DBLP:journals/ws/KharlamovHSBJXS17,DBLP:conf/semweb/SequedaBMH19}. In such a case, an  
ontology of rather small size may suffice and deriving it from a collection of examples seems natural, close in spirit to query-by-example.

As ontology languages, we concentrate on the 
expressive yet fundamental DLs  \ALC and \ALCI, and
as query languages we consider atomic queries (AQs), conjunctive queries (CQs), 
full CQs (CQs without quantified
variables), and unions of conjunctive queries (UCQs). In addition, we study a fitting problem
in which the examples only consist of an ABox and where we seek an ontology that is consistent with the positive examples and inconsistent with the negative ones; this is related, but not identical to both AQ-based and full CQ-based fitting.
For all of the resulting combinations, we
provide effective characterizations and 
determine the precise complexity of deciding whether a fitting ontology exists. The
algorithms that we use to prove the upper
bounds are able to produce explicit
fitting ontologies.

For consistency-based fitting and for AQs,
our characterizations of fitting existence make use of
the connection between ontology-mediated querying and constraint satisfaction problems (CSPs) established in \cite{DBLP:journals/tods/BienvenuCLW14}. While this connection
does not extend to full CQs, the intuitions do and in all  three cases our characterizations enable a 
\coNPclass upper bound, both for \ALC- and \ALCI-ontologies. Corresponding lower bounds are easy to obtain by a reduction from the digraph homomorphism problem. We remark that
the complexity is thus much lower than that of the associated query entailment problems, meaning to decide whether $\Amc \cup \Omc \models q$ for a given ABox \Amc, ontology \Omc, and query $q$. In fact, the complexity
of query entailment is \ExpTime-complete for all cases discussed so far \cite{Baader2017}.

For CQs
and UCQs, we give a characterization of fitting existence based on the existence of certain forest models \Imc. These models are potentially infinite, intuitively because the positive examples $(\Amc,q)$ act similarly to an existential rule: if we 
homomorphically find \Amc  in \Imc, then at the same place we must (in a certain, slightly unusual sense) also find~$q$. Thus the 
existential quantifiers of $q$ may enforce
that every element of \Imc has a successor, resulting in infinity. 
As a consequence of this effect, the computational complexity of fitting existence for CQs and UCQs turns out to be much higher than for AQs and full CQs: it is \TwoExpTime-complete both for CQs
and UCQs, no matter whether we want to fit
an \ALC- or an \ALCI-ontology. For \ALCI,
the complexity thus coincides with that of
query entailment, which is \TwoExpTime-complete both for CQs and UCQs \cite{DBLP:conf/dlog/Lutz08}. For \ALC,
the complexity of the fitting problem is
harder than that of the associated
entailment problems, which are both \ExpTime-complete \cite{DBLP:conf/dlog/Lutz08}. 
Our upper bounds are obtained by
a mosaic procedure. The lower bounds for
\ALCI are proved by reduction from 
query entailment and for \ALC they are 
proved by reduction from the word problem of exponentially space-bounded alternating Turing machines.

Proofs are provided in the appendix.

\paragraph{Related Work.} To the best of our knowledge, the only other study of fitting problems for ontologies is a recent one by \citet{SimonKR25}. However, it uses interpretations as examples rather than ABox and queries. Vaguely related are fitting problems for DL concepts. These have been investigated from
a  practical angle by \cite{DBLP:journals/ml/LehmannH10,DBLP:conf/www/BuhmannLWB18,DBLP:journals/fgcs/RizzoFd20}, and from a foundational perspective
by  \cite{DBLP:conf/ijcai/FunkJLPW19,DBLP:conf/kr/JungLPW20,DBLP:conf/kr/JungLPW21,DBLP:journals/ai/JungLPW22}. Other approaches
that support the construction of an
entire ontology include Angluin's
framework of exact learning by \cite{DBLP:journals/jmlr/KonevLOW17} and formal
concept analysis  by \cite{DBLP:conf/ijcai/BaaderGSS07,DBLP:conf/icfca/BaaderD09,DBLP:conf/aaai/Kriegel24}.
These and related approaches are
surveyed by \citet{DBLP:journals/ki/Ozaki20}.

\section{Preliminaries}

\subsection*{Description Logic} %

Let \NC, \NR, and \NI be countably infinite sets of \emph{concept names}, \emph{role names}, and \emph{individual names}.
An \emph{inverse role} takes the form $r^-$ with $r$ a role name, and a \emph{role} is a role name or an inverse role. If $r=s^-$ is an inverse role, then we set $r^-=s$.
An \emph{\ALCI-concept} $C$ is built according to 
$$
C,D ::= \top \mid A \mid \neg C \mid C \sqcap D \mid \exists r . D
$$
where $A$ ranges over concept names and $r$ over roles. 
As usual, 
we write $\bot$ as abbreviation for $\neg \top$, $C \sqcup D$ for $\neg(\neg C \sqcap \neg D)$, and $\forall r . C$ for $\neg \exists r. \neg C$. 
An \emph{\ALC-concept} is an \ALCI-concept that does not use inverse
roles.  

An \emph{\ALCI-ontology} is a finite set of \emph{concept inclusions
  (CIs)} $C \sqsubseteq D$, where $C,D$ are \ALCI-concepts. 
 \emph{\ALC-ontologies} are defined likewise.
 We may write $C \equiv D$ as shorthand for $C \sqsubseteq D$ and
 $D \sqsubseteq C$.
An \emph{ABox} is a finite set of
\emph{concept assertions} $A(a)$ and \emph{role assertions} $r(a,b)$
where $A$ is a concept name, $r$ a role name, and $a,b$ are individual
names. We use $\mn{ind}(\Amc)$ to denote the set of individual names used
in \Amc. 

The semantics of concepts is defined as usual in terms of interpretations
$\Imc=(\Delta^\Imc,\cdot^\Imc)$ with $\Delta^\Imc$ the (non-empty)
\emph{domain} and $\cdot^\Imc$ the \emph{interpretation function}, we
refer to \cite{Baader2017} for full details.  An interpretation \Imc
\emph{satisfies} a CI $C \sqsubseteq D$ if $C^\Imc \subseteq D^\Imc$,
a concept assertion $A(a)$ if $a \in A^\Imc$, and
a role assertion $r(a,b)$ if $(a,b) \in r^\Imc$; we thus make the \emph{standard
  names assumption}.  We say that \Imc is a \emph{model} of an
ontology \Omc, written $\Imc \models \Omc$, if it satisfies all concept inclusions
in it, and likewise for ABoxes. An ontology is \emph{satisfiable} if it
has a model and an ABox is \emph{consistent} with an ontology \Omc if \Amc and \Omc 
have a common model.


A \emph{homomorphism} from an interpretation $\Imc_1$ to an interpretation $\Imc_2$
is a mapping $h \colon \Delta^{\Imc_1} \to \Delta^{\Imc_2}$
such that $d \in A^{\Imc_1}$
implies $h(d) \in A^{\Imc_2}$ and $(d,e) \in r^{\Imc_1}$ implies $(h(d),h(e)) \in r^{\Imc_2}$
for all concept names $A$, role names~$r$, and $d,e \in \Delta^{\Imc_1}$. We write $\Imc_1 \to \Imc_2$ if there exists
a homomorphism from $\Imc_1$ to $\Imc_2$ and $\Imc_1 \not\to \Imc_2$ otherwise. We will also use
homomorphisms from ABoxes to ABoxes and from ABoxes to interpretations. These are defined as expected.
In particular, homomorphisms from ABox to ABox need not map individual names to themselves, which would trivialize them.

\subsection*{Queries} 

A \emph{conjunctive query (CQ)} takes the form
$q = \exists \overline x \, \varphi(\overline x)$
where $\overline x$ is a tuple of variables and $\varphi$
 a
conjunction of \emph{atoms} $A(t)$ and $r(t,t')$, with
\mbox{$A \in \NC$}, $r \in \NR$, and $t,t'$ variables from
$\overline x$ or individuals from \NI.
%
With $\mn{var}(q)$, we denote the set of variables in $\overline x$.
We take the liberty to view
$q$ as a set of atoms, writing e.g.\ $\alpha \in
q$ to indicate that $\alpha$ is an atom in~$q$. 
We may also write $r^-(x,y) \in q$ in place of \mbox{$r(y,x) \in q$}.
 An \emph{atomic query (AQ)}
is a  CQ of the simple form $A(a)$, with $A$ a concept name.
A CQ is \emph{full} if it does not contain any existentially quantified variables.
A {\em union of conjunctive queries (UCQ)} $q$ is a disjunction of CQs.
We refer to each of these classes of queries as a \emph{query language}.


 A CQ $q$ gives rise to an interpretation
 $\Imc_q$ with $\Delta^{\Imc_q}$ the set of all variables and individuals in $q$, $A^{\Imc_q} = \{ t \mid
 A(t) \in q \}$, and $r^{\Imc_q}=\{(t,t') \mid r(t,t') \in
 q\}$ for all $A \in \NC$ and $r \in \NR$.
With a homomorphism from a CQ $q$ to an
interpretation \Imc, we mean a homomorphism from $\Imc_q$ to \Imc
 that is the identity
on all individual names. If we want to emphasize
the latter property, we may speak of a \emph{strong} homomorphism. In contrast, a
\emph{weak} homomorphism from $q$ to \Imc, as
sometimes used in our proofs, need
not be the identity on individual names.
For an interpretation \Imc and a UCQ $q$, we write
$\Imc \models q$ if there is a (strong) homomorphism $h$ from a CQ in $q$ to \Imc. 
For an ABox \Amc and ontology \Omc, we write $\Amc \cup \Omc \models q$ if  $\Imc \models q$ for all models
\Imc of $\Amc$ and~\Omc.

Note that all queries introduced above are Boolean, that is, they evaluate to true or false instead of producing answers. For the purposes of this paper, however, this is without loss of generality since we
admit individual names in queries. 

We use $||O||$ to denote the \emph{size} of any syntactic object $O$ such as a concept, an ontology, or a query. It is defined as the length
of the encoding of $O$ as a word over some suitable alphabet.

An \emph{\ALC-forest model} \Imc of an ABox \Amc is a model of \Amc such that
\begin{enumerate}
\item the directed graph
$
  (\Delta^\Imc, \bigcup_r r^\Imc \setminus (\mn{ind}(\Amc) \times \mn{ind}(\Amc))
$
is a forest (a disjoint union of trees) and 

\item $r^\Imc \cap (\mn{ind}(\Amc) \times \mn{ind}(\Amc)) = \{ (a,b) \mid r(a,b) \in \Amc \}$. 

\end{enumerate}
\emph{\ALCI-forest models}
are defined likewise, but based on 
the undirected version of the graph in Point~1. In other words, in \ALC-forest models all edges must point away from the roots of the trees while this is not the case for \ALCI-forest models. 
With the \emph{degree} of an interpretation, we mean the maximal number of neighbors of any element in its domain.
\begin{lemma}
    \label{lem:forest}
    Let $\Lmc \in \{ \ALC,\ALCI \}$,
    \Omc be an \Lmc-ontology, \Amc an ABox, and $q$ a UCQ. If $\Amc \cup \Omc \not \models q$, then there is an \Lmc-forest model \Imc of
    \Amc and \Omc of degree at most $||\Omc||$ that satisfies $\Imc \not\models q$.
\end{lemma}
The proof of Lemma~\ref{lem:forest},
which can be found for instance in \cite{DBLP:conf/dlog/Lutz08}, relies on unraveling, which we shall also use in this article. Let \Imc be an interpretation and $d \in \Delta^\Imc$.
A \emph{path} in \Imc is a sequence
$p=d_1 r_1 \cdots d_{n - 1} r_{n - 1} d_n$ of domain elements $d_i$ from $\Delta^\Imc$ and role names $r_i$
such that $(d_i, d_{i + 1}) \in r_{i}^{\Imc}$ for  $1 \leq i < n$.
We say that the path \emph{starts} at $d_1$ and use 
$\mn{tail}(p)$ to denote $d_n$. 

The \emph{$\ALC$-unraveling} of \Imc at $d$ is the
interpretation \Umc defined as follows:
%
$$
\begin{array}{rcl}
    \Delta^\Umc &=& \text{set of all paths in \Imc starting at } d \\[1mm]
    A^\Umc &=& \{ p \mid \mn{tail}(p) \in A^\Imc \} \\[1mm]
    r^\Umc &=& \{ (p,p') \mid p' = pre \text{ for some } e \}.
\end{array}
$$

The $\ALCI$-unraveling of $\Imc$ at $d$ is defined likewise, with the
modification that inverses of roles can also appear in paths and that $(p, p')$
is also included in $r^{\Umc}$ if $p = p' r^- e$.
Note that there is a homomorphism from $\Umc$ to \Imc that maps every $p \in \Delta^{\Umc}$ to $\mn{tail}(p)$.

\subsection*{ABox Examples and the Fitting Problem} 


Let \Qmc be a query language such as $\Qmc = \text{AQ}$ or $\Qmc = \text{CQ}$. An \emph{ABox-\Qmc example} is a pair $(\Amc,q)$
with \Amc an ABox and $q$ a query from \Qmc such that all individual names that appear in $q$ are from $\mn{ind}(\Amc)$.

By a \emph{collection of labeled examples} we  mean a pair 
$E=(E^+, E^-)$ of finite sets of examples. The examples in $E^+$
are the \emph{positive examples} and the examples in $E^-$ are the \emph{negative
examples}. We say that \Omc \emph{fits} $E$ if 
$\Amc \cup \Omc
\models q$ for all $(\Amc,q) \in E^+$ and $\Amc \cup \Omc
\not\models q$ for all $(\Amc,q) \in E^-$. The
following example illustrates this central
notion.
\begin{example}
Consider the collection of labeled ABox-UCQ examples $E = (E^+, E^-)$,
where
\begin{align*}
    E^+ = \{ \ & (\{ \mn{authorOf}(a, b), \mn{Publication}(b) \},\mn{Author}(a)), \\
    & (\{ \mn{Reviewer}(a) \},\exists x\, \mn{reviews}(a, x) \land \mn{Publication}(x) ), \\
    & (\{ \mn{Publication}(a) \},\mn{Confpaper}(a) \vee
    \mn{Jarticle}(a)) \ \},
\end{align*}
and $E^-=\emptyset$.
An \ALC-ontology that fits $(E^+, E^-)$ is
$$
\begin{array}{r@{\;}r@{\;}c@{\;}l}
\Omc = \{ 
  &\exists \mn{authorOf}. \mn{Publication} &\sqsubseteq& \mn{Author} \\[1mm]
 &\mn{Reviewer} &\sqsubseteq& \exists \mn{reviews}. \mn{Publication} \\[1mm] 
 &\mn{Publication} &\sqsubseteq&  \mn{Confpaper} \sqcup \mn{Jarticle} \ \}.  
\end{array}
$$
There are, however, many other fitting \ALC-ontologies as well,
including as an extreme $\Omc_\bot = \{ \top \sqsubseteq \bot \}$ and, say,
$$\Omc' = \Omc \cup \{ \mn{Author} \sqsubseteq \exists \mn{authorOf} . \mn{Reviewer} \}.$$
%
We can make both of them  non-fitting by adding the negative example
$$(\{ \mn{Author}(a) \}, \exists x \, \mn{authorOf}(a,x) \wedge \mn{Reviewer}(x) ).$$
\end{example}


Let \Lmc be an ontology language, such as $\Lmc = \ALCI$, and  \Qmc a query
language. Then \emph{(\Lmc,\Qmc)-ontology fitting} is the problem
to decide, given as 
input a collection of 
labeled \mbox{ABox-\Qmc} examples $E$, whether $E$ admits a fitting \Lmc-ontology.
We generally assume that the ABoxes used in $E$ 
 have pairwise disjoint sets of individual names.
It is not hard to verify that this is without loss of generality because consistently renaming individual names in a collection of examples has no impact on the existence of a fitting ontology.

There is a natural variation of 
$(\Lmc,\Qmc)$-ontology fitting 
where one additionally requires the fitting ontology to be consistent with all ABoxes that occur in positive examples.\footnote{Note that it is
implicit already in the original 
formulation that the fitting ontology must be consistent with all ABoxes that occur in negative examples $(\Amc,q)$, as otherwise $\Amc \cup \Omc \models q$.} We then speak of 
\emph{consistent $(\Lmc,\Qmc)$-ontology fitting}. The following observation shows that it suffices to 
design algorithms for  $(\Lmc,\Qmc)$-ontology fitting  as originally introduced.
\begin{proposition}
\label{prop:fittingswithconsistentABoxes}
  Let \Lmc be any ontology language and $\Qmc \in \{ \text{AQ}, \text{FullCQ}, \text{CQ}, \text{UCQ} \}$. Then there is a polynomial time reduction from consistent $(\Lmc,\Qmc)$-ontology fitting 
  to  $(\Lmc,\Qmc)$-ontology fitting.
\end{proposition}
\noindent
\begin{proof}
    We exemplarily treat the case $\Qmc = \text{AQ}$. The other cases are similar.
    Let $E$ be a collection of labeled ABox-AQ examples. We extend $E$ to a collection $E'$ by adding, for each positive example $(\Amc,Q(a)) \in E^+$,
    a negative example $(\Amc,X(a))$ where $X$ is a concept name that is not mentioned in $E$. Then $E$ admits a fitting \Lmc-ontology that
    is consistent with all ABoxes in positive examples if and only if 
    $E'$ admits a fitting \Lmc-ontology.
    In fact, any \Lmc-ontology that fits $E$, does not mention $X$, and is consistent with all ABoxes in positive examples is also a fitting of $E'$. 
    Conversely, any ontology that fits $E'$ must be
    consistent with all ABoxes that occur in positive examples as otherwise one of the additional negative examples would be violated. 
\end{proof}



\section{Consistency-Based Fitting}
\label{sect:consfit}

We start with a version of ontology fitting that is based on ABox consistency rather than on querying.
An example is then simply an ABox, and
 an ontology \Omc \emph{fits} a collection
 of examples $E=(E^+,E^-)$ if \Amc is consistent with \Omc
for all $\Amc \in E^+$ and inconsistent with \Omc for all
$\Amc \in E^-$. We refer to the induced 
decision problem as \emph{consistent \Lmc-ontology fitting}.\footnote{Not to be confused with consistent $(\Lmc,\Qmc)$-ontology fitting as briefly considered in Proposition~\ref{prop:fittingswithconsistentABoxes}.} We believe that it is natural to consider this basic case as a warm-up. 
%
\begin{example}
    Consider the collection of labeled ABox examples $E = (E^+, E^-)$,
    where
    \begin{itemize}

      \item $E^+$ contains the ABox $\Amc_1 = \{ r(a_1, a_2) \}$ and
    \item $E^-$ contains 
      the ABox $\Amc_2 = \{ r(b, b) \}$.

    \end{itemize}
    Then $\Omc = \{ \exists r. \exists r. \top \sqsubseteq \bot \}$ is an $\ALC$-ontology that fits $(E^+, E^-)$.
    If we swap $E^+$ and $E^-$, then there is no fitting $\ALC$-ontology or $\ALCI$-ontology.
\end{example} 
We start with a characterization 
of consistent \Lmc-ontology fitting in 
terms of homomorphisms.
\begin{restatable}{theorem}{thmconschar}
\label{thm:conschar}
	Let $E = (E^+, E^-)$ be a collection of labeled ABox examples, $\Lmc \in \{ \ALC,\ALCI \}$, and $\Amc^+ = \biguplus E^+$. Then  the following 
    are equivalent:
    \begin{enumerate}

        \item     $E$ admits a fitting
    \Lmc-ontology;

    \item
        $\mathcal{A} \not\rightarrow \Amc^{+}$ for all $\Amc \in E^{-}$.
    \end{enumerate}
\end{restatable}   
Note that the characterizations for \ALC and \ALCI are identical,
and thus a collection of labeled ABox-consistency examples admits a fitting \ALC-ontology if and only if it admits a fitting \ALCI-ontology. It is clear from the proofs that further adding role inclusions, see \cite{Baader2017}, does not increase the separating 
power either. Adding number restrictions,
however, has an impact; see Section~\ref{sect:concl}.




The proof of Theorem~\ref{thm:conschar} makes use of the connection between ontology-mediated
querying and constraint satisfaction problems (CSPs) established in 
\cite{DBLP:conf/kr/LutzW12}. 
In particular, for the ``$2 \Rightarrow 1$'' direction we use the fact that 
for every ABox \Amc, one can construct an ontology \Omc such that for all ABoxes \Bmc that only use concept and role names from~\Amc, the following
holds: $\Bmc \rightarrow \Amc$ if and only if \Bmc is consistent with~$\Omc$.
We apply this choosing $\Amc=\Amc^+$.

We  obtain   an upper
bound for consistent ontology fitting by a straightforward
implementation of Point~2 of Theorem~\ref{thm:conschar} and
a corresponding lower bound by an easy reduction of the homomorphism problem for directed graphs. 
\begin{restatable}{theorem}{thmconsnPcompl}
    Let $\Lmc \in \{ \ALC,\ALCI \}$. Then consistent \Lmc-ontology fitting is \coNPclass-complete.
\end{restatable}
%
%
It might be worthwhile to point out as a corollary of Theorem~\ref{thm:conschar} that negative examples can be
treated independently, in the following sense. 
\begin{corollary}
\label{cor:onlyonenegative}
    Let $\Lmc \in \{\ALC,\ALCI\}$  and $E$ be a collection of labeled ABox examples, with $E^-=\{\Amc_1,\ldots,\Amc_n\}$. Then $E$ admits a fitting \Lmc-ontology if and only if for $1 \leq i \leq n$, the collection of ABox examples $(E^+,\{\Amc_i\})$ admits a fitting \Lmc-ontology.
\end{corollary}
We note that, in related fitting settings such as the
one studied in \cite{DBLP:conf/ijcai/FunkJLPW19}, statements of this form can often be 
shown in a very  direct way rather than via a characterization. This does not appear to be the case here.

\section{Atomic Queries}
\label{sect:aqs}

We  consider atomic queries and again present a characterization
in terms of homomorphisms. These
are now in the other direction, from the positive examples to
the negative examples, corresponding to the 
complementation involved in the well-known reductions from ABox consistency to AQ entailment and vice versa.
 What is more important, however, is that
it does no longer suffice to work directly with the (negative) examples. In fact, the positive examples act like a form of implication
on the negative examples, 
similarly to an existential
rule (with atomic unary rule head),
and as a result we must first suitably
enrich the negative examples.


		Let $E = (E^+, E^-)$ be a collection of labeled ABox-AQ examples. A
            \emph{completion} for $E$ is an 
            ABox \Cmc that extends the ABox $\Amc^- :=\biguplus_{(\Amc,Q(a)) \in E^-} \Amc$ by concept assertions $Q(b)$ where  $b \in \mn{ind}(\Amc^-)$ and $Q$ a concept name that occurs as an AQ in $E^+$.
\begin{restatable}{theorem}{thmAQchar}
\label{thm:AQchar}
	Let $E = (E^+, E^-)$ be a collection of labeled ABox-AQ examples
    and let $\Lmc \in \{ \ALC,\ALCI \}$. 
    Then  the following 
    are equivalent:
    \begin{enumerate}

        \item     $E$ admits a fitting
    \Lmc-ontology;

    \item there is a completion \Cmc for $E$ such that 
    \begin{enumerate}
        
    \item 
    for all \mbox{$(\Amc,Q(a)) \in E^{+}$}:
       if $h$ is a homomorphism from \Amc to \Cmc, then \mbox{$Q(h(a)) \in \Cmc$};

   \item for all $(\Amc,Q(a)) \in E^-$: $Q(a) \not\in \Cmc$.

    \end{enumerate}
    \end{enumerate}
\end{restatable}   
The announced behavior of positive examples as an implication is reflected by Point~2a.
Note that, as in the consistency case, there is no
difference between \ALC and \ALCI.
The proof of Theorem~\ref{thm:AQchar} is
similar to that of Theorem~\ref{thm:conschar}. It
might be worthwhile to note that Theorem~\ref{thm:AQchar}
does not suggest a counterpart of Corollary~\ref{cor:onlyonenegative}. 
Such a counterpart would speak about
single positive examples rather than 
single negative ones, because of the
complementation mentioned above. 
The following
example illustrates that it does not suffice to concentrate on a
single positive example (nor a single negative one).
Intuitively, this is due to the fact that the ABox \Amc in Point~2a
may be disconnected.
\begin{example}\label{ex:AQ}
  Consider the collection of labeled ABox-AQ examples $E= (E^+, E^-)$ with
  \begin{align*}
      E^+ &= \{\; (\{ A_2(a)\}, A_1(a)), \ (\{A_3(b), A_4(b')\},
            A_2(b)) \;\}\\
      E^- &= \{ \; (\{A_3(c)\}, A_1(c)),\ (\{A_4(d)\}, A_5(d)) \; \}. 
  \end{align*}
  $E$ does  not admit a fitting \ALCI-ontology, which can be seen
  by applying Theorem~\ref{thm:AQchar}: by definition, any completion \Cmc must satisfy
  $\Amc^- = \{ A_3(c), A_4(d) \} \subseteq \Cmc$. To satisfy Condition~2a of
  Theorem~\ref{thm:AQchar},
  it must then  also satisfy $A_2(c) \in \Cmc$ and $A_1(c) \in \Cmc$. But
  then
  \Cmc violates Condition~2b of Theorem~\ref{thm:AQchar} for the first
  negative
  example. If we drop any of the positive examples, we find completions $\Cmc = \Amc^-$ and $\Cmc = \Amc^- \cup \{B(c)\}$, respectively, which satisfy Conditions~2a and~2b. Also dropping any negative example leads to a satisfying completion.
\end{example}

In contrast to the case of consistent \Lmc-ontology fitting, a naive implementation of the
characterization given in Theorem~\ref{thm:AQchar} only gives a $\Sigma^p_2$-upper bound:  guess the completion \Cmc for $E$ and then co-guess the homomorphisms in Point~(a).
In the following, we show how to improve this to \coNPclass. The main observation is that we
can do better than guessing \Cmc blindly, by
treating positive examples as rules. 
%
\begin{definition}
\label{def:refcand}
Let $E = (E^+, E^-)$ be a collection of labeled ABox-AQ examples, and let
$\Amc^- = \biguplus_{(\Amc,Q(a)) \in E^-} \Amc$. A \emph{refutation
candidate} for $E$ is an 
            ABox \Cmc that can be obtained by starting with $\Amc^-$ and then
         applying the following rule zero or more times:
    \begin{description}

      \item[(\Rsf)]  
        if $(\Amc,Q(a)) \in E^+$ and $h$ is a homomorphism from \Amc
        to \Cmc, then set $\Cmc = \Cmc \cup \{ Q(h(a)) \}$.
            \end{description}

 \end{definition}
Rule (\Rsf) can add at most $|E^+|\cdot|\mn{ind}(\Amc^-)|$ (and thus only polynomially many) assertions to $\Amc^-$.
\begin{restatable}{proposition}{proprefcand}
\label{prop:refcand}
    	Let $E = (E^+, E^-)$ be a collection of labeled ABox-AQ examples
    and let $\Lmc \in \{ \ALC,\ALCI \}$. 
    Then  the following 
    are equivalent:
    \begin{enumerate}

        \item     $E$ admits no fitting
    \Lmc-ontology;

    \item there is a refutation candidate \Cmc for $E$ such that 
     $Q(a) \in \Cmc$ for some $(\Amc,Q(a)) \in E^-$.
       
    \end{enumerate}

\end{restatable}
Note that Point~1 of Proposition~\ref{prop:refcand} is the
complement of Point~1 of Theorem~\ref{thm:AQchar},
and thus \Cmc has a different role: we may co-guess
it, in contrast to the \Cmc in  Theorem~\ref{thm:AQchar} which needs to be
guessed. A close look reveals that we indeed
obtain  a \coNPclass upper bound.
\begin{theorem}
\label{thm:AQincoNP}
    Let $\Lmc \in \{ \ALC,\ALCI \}$. Then  $(\Lmc,\text{AQ})$-ontology fitting  is \coNPclass-complete.
\end{theorem}
\noindent
\begin{proof}
   \coNPclass-hardness can be proved as in the consistency case. It thus remains to
   argue that the complement of the $(\Lmc,\text{AQ})$-ontology fitting problem is in \NPclass. By Proposition~\ref{prop:refcand}, it suffices to guess 
   an ABox \Cmc with the same individuals as $\Amc^-$ and to verify that (i)~\Cmc is a refutation candidate and (ii)~$Q(a) \in \Cmc$ for some $(\Amc,Q(a)) \in E^-$. To verify (i)~we may guess, along with \Cmc, a sequence of positive examples $(\Amc,Q(a)) \in E^+$ with associated homomorphisms from \Amc that demonstrate the construction of \Cmc from $\Amc^-$ by repeated applications of Rule~(\Rsf). The maximum length of the sequence is $|E^+|\cdot|\mn{ind}(\Amc^-)|$. With the sequence at hand, it is then easy to verify deterministically in polynomial time that \Cmc is a refutation candidate.
\end{proof}
In view of the close connection between
ABox consistency and AQ entailment, one may
wonder whether the two fitting 
problems studied in this and the preceding
section are, in some reasonable sense, identical. 
 We may
ask whether for every instance $E$ of consistent \Lmc-ontology fitting, there is an instance $E'$
of $(\Lmc,AQ)$-ontology fitting with the same set of fitting ontologies and vice versa. 
It turns out that neither is the case. For better
readability, in the following we refer to ABox examples as ABox-consistency examples. For \Lmc an ontology language and $E$ a collection of labeled  examples, let $O_{E, \Lmc}$ be the set of all $\Lmc$-ontologies that fit $E$.
\begin{restatable}{proposition}{propexppowerone}
Let $\Lmc \in \{\ALC, \ALCI\}$. Then 
\begin{enumerate}
	\item there exists a collection of ABox-AQ examples $E$, such that there is no collection of ABox-consistency examples $E'$ with $O_{E, \Lmc} = O_{E', \Lmc}$;
    \item there exists a collection of ABox-consistency examples, such that there is no collection of ABox-AQ examples $E'$ with $O_{E, \Lmc} = O_{E', \Lmc}$.
\end{enumerate}
\end{restatable}
\noindent
\begin{proof}
For Point~1, consider the collection of ABox-AQ examples $E = ( E^+, E^-)$ with $E^+ = \{ ( \{ A(a) \}, B_1(a ) ) \}$ and
$E^- = \{ ( \{ A(a) \} , B_2(a) ) \}$. Note that we use $A(a)\}$ only to ensure that $a$ occurs in the ABoxes.
It is easy to see that for $\Omc_1 = \{ \top \sqsubseteq B_1 \}$ and $\Omc_2 = \{ \top \sqsubseteq B_2\}$, $\Omc_1$ fits $E$ and $\Omc_2$ does not.
Furthermore, every ABox is consistent with both $\Omc_1$ and $\Omc_2$.
Therefore, for every collection of ABox-consistency example $E'$,
either $\{ \Omc_1, \Omc_2 \} \subseteq O_{E', \Lmc}$ or $\Omc_1 \notin O_{E', \Lmc}$ and $\Omc_2 \notin O_{E', \Lmc}$.

\smallskip

For Point~2, consider the collection of ABox-consistency examples $E = (E^+, \emptyset )$ with $E^+ = \{ \{ s(a, b)\} \}$ and let $E'$ be any collection of ABox-AQ examples.
If there is a negative example $(\Amc, A(a))$ in $E'$ for some concept name $A$,
then for $\Omc = \{ \top \sqsubseteq A \}$, $\Omc$ fits $E$ but $\Omc$ does not fit $E'$.
If there is no negative example in $E'$, then for $\Omc' = \{ \top \sqsubseteq \bot \}$, 
$\Omc'$ fits $E'$, but $\Omc'$ does not fit $E$.
\end{proof}

\section{Full Conjunctive Queries}
\label{sect:fullCQ}

We next study the case of full conjunctive queries. Technically, it is closely related
to both the AQ-based case and the ABox consistency-based case. However, the potential presence of role atoms in queries brings some technical complications.

We call an example $(\Amc,q)$ \emph{inconsistent} if any \ALCI-ontology $\Omc$ that satisfies $\Amc \cup \Omc \models q$ is inconsistent with~\Amc,
and \emph{consistent} otherwise. It is easy to see
that any example $(\Amc,q)$ such that $q$ contains
a role atom $r(a,b) \notin \Amc$ must be inconsistent. In fact, this follows from Lemma~\ref{lem:forest}. Conversely, any example 
$(\Amc,q)$ such that $q$ does not contain
such a role atom is consistent. This is witnessed
by the ontology 
$$
  \Omc = \{ \top \sqsubseteq A \mid A(a) \in q \}.
$$
%
%
%
Note that an inconsistent positive example $(\Amc,q)$
expresses the constraint that \Amc must be inconsistent with
the  fitting ontology \Omc and an inconsistent negative example $(\Amc,q)$
expresses  that \Amc must be consistent with
the  fitting ontology~\Omc. 
In view of this, it is clear that, up to swapping positive and negative examples, FullCQ-based
fitting generalizes consistency-based fitting.
Moreover, it trivially generalizes AQ-based
fitting since every AQ is a full CQ.
\begin{proposition}
Let $\Lmc \in \{\ALC, \ALCI\}$. Then for every collection of ABox-consistency or ABox-AQ examples $E$, there is a collection of ABox-FullCQ examples such that $O_{E, \Lmc} = O_{E', \Lmc}$.
\end{proposition}
For the following development, we would ideally
like to get rid of inconsistent examples to 
achieve simpler characterizations. 

There is, however, no obvious way to achieve this for inconsistent positive examples. We can get rid of 
inconsistent negative examples based on the following observation. Assume that a collection
of ABox-FullCQ examples $E$ contains an inconsistent negative example $(\Amc,q)$.
We replace it with the negative example $(\Amc,X(a))$ where $X$ is a
fresh concept name and $a \in \mn{ind}(\Amc)$ is chosen arbitrarily.
 The set of fitting \ALCI-ontologies for the resulting set of examples $E'$ remains essentially the same.
\begin{restatable}{lemma}{lemNoInconsEx}
  For $\Lmc \in \{ \ALC,\ALCI \}$, there
  is an \Lmc-ontology that fits $E$ if and only if there is one that fits $E'$.
\end{restatable}

\medskip
\emph{Completions} for collections of ABox-FullCQ examples are defined in 
exact analogy with completions for collections of ABox-AQ examples.
%
\begin{restatable}{theorem}{thmfullCQchar}
\label{thm:fullCQchar}
	Let $E = (E^+, E^-)$ be a collection of labeled ABox-FullCQ examples
    and let $\Lmc \in \{ \ALC,\ALCI \}$. 
    Then  the following 
    are equivalent:
    \begin{enumerate}

        \item     $E$ admits a fitting
    \Lmc-ontology;

    \item there is a completion \Cmc for $E$ such that 
    \begin{enumerate}
        
    \item 
    for all consistent \mbox{$(\Amc,q) \in E^{+}$}:
       if $h$ is a homomorphism from \Amc to \Cmc and $Q(a) \in q$, then  \mbox{$Q(h(a)) \in \Cmc$};

   \item for all 
   $(\Amc,q) \in E^-$: there is a $Q(a) \in q$ such that $Q(a) \not\in \Cmc$;

    \item for all inconsistent \mbox{$(\Amc,q) \in E^{+}$}: there is no 
    homomorphism from \Amc to \Cmc.
   
    \end{enumerate}
    \end{enumerate}
\end{restatable} 
The proof of Theorem~\ref{thm:fullCQchar}
uses Theorem~\ref{thm:AQchar}. Note that, once more, there is no difference between \ALC and \ALCI. As in the case of AQs,
our characterization suggests only a $\Sigma^p_2$
upper bound. However, we can  get down to \coNPclass in the same way as for AQs.
The following is in exact analogy with Definition~\ref{def:refcand}.
\begin{definition}
Let $E = (E^+, E^-)$ be a collection of labeled ABox-FullCQ examples, and let
$\Amc^- = \biguplus_{(\Amc,q) \in E^-} \Amc$. A \emph{refutation
candidate} for $E$ is an 
            ABox \Cmc that can be obtained by starting with $\Amc^-$ and then
         applying the following rule zero or more times:
    \begin{description}

      \item[(\Rsf)]  
        if $(\Amc,q) \in E^+$ is consistent and $h$ is a homomorphism from \Amc
        to \Cmc and $Q(a) \in q$, then set $\Cmc = \Cmc \cup \{ Q(h(a)) \}$.
            \end{description}
 \end{definition}
The proof of the following is then analogous to that of Proposition~\ref{prop:refcand}. Details are omitted.
\begin{proposition}
\label{prop:refcandFullCQ}
    	Let $E = (E^+, E^-)$ be a collection of labeled ABox-AQ examples
    and let $\Lmc \in \{ \ALC,\ALCI \}$. 
    Then  the following 
    are equivalent:
    \begin{enumerate}

        \item     $E$ admits no fitting
    \Lmc-ontology;

    \item there is a refutation candidate \Cmc for $E$ such that 
    one of the following conditions is satisfied:
    \begin{enumerate}
        
     \item there is an  $(\Amc,q) \in E^-$ such that $Q(a) \in \Cmc$ for all $Q(a) \in q$;

     \item there is an inconsistent $(\Amc,q)\in E^+$ and a homomorphism from \Amc to \Cmc.
     
           \end{enumerate}
    \end{enumerate}
\end{proposition}
And finally, the proof of the following is similar 
to that of Theorem~\ref{thm:AQincoNP}.

\begin{restatable}{theorem}{thmfullCQconpupper}
    Let $\Lmc \in \{ \ALC,\ALCI \}$. Then  $(\Lmc,\text{FullCQ})$-ontology fitting  is \coNPclass-complete.
\end{restatable}

\section{CQs and UCQs}

We now turn to conjunctive queries and UCQs, which
constitute the most challenging case. This is due
to the fact that, since positive examples act
as implications, the presence of existentially quantified variables in the query effectively 
turns these examples into a form of existential rule.
Thus, completions as used for AQs and full CQs
are no longer finite.

Throughout this section, we assume that ABoxes in positive examples are never empty. This is mainly
to avoid dealing with too many special cases in
the technical development. We conjecture that admitting empty ABoxes does not change the obtained results.

\subsection{Characterization for \ALC and \ALCI}





We start with a characterization for the case of UCQs (and thus also CQs) that is
similar in spirit to the one for full CQs given
in Theorem~\ref{thm:fullCQchar}. The characterization applies to both
\ALC and \ALCI in a uniform, though not identical way. As already mentioned, finite completions no longer suffice
and we replace them with potentially infinite
interpretations. There is another interesting view on this: the fitting
ontologies constructed (as part of the proofs) in
Sections~\ref{sect:consfit} and~\ref{sect:aqs} do
not make existential statements, that is, their sets
of models are closed under taking induced subinterpretations. This, however, cannot be
achieved
for CQs and UCQs. We illustrate this by the
following example which also shows that, unlike
for AQs and full CQs, there is a difference between
fitting \ALC-ontologies and fitting \ALCI-ontologies. Induced subinterpretations are defined
in exact analogy with induced substructures in model theory.
\begin{example}
  Consider the collection of ABox-CQ examples
  $E=(E^+,E^-)$ where
  $$
  \begin{array}{rcc@{}l}
    E^+ &=& \{ & (\{A_1(a)\}, \exists x\, r(x,a) \land A_2(x)), \\[1mm]
    &&& ( \{A_2(a)\}, \exists x \, r(x,a) \land A_1(x)) \ \} \\[1mm]
    E^- &=& \{ &(\{A_1(a) \} ,B(a)), \ (\{A_2(a) \} ,B(a)) \}.
  \end{array}
  $$
  Then there is a fitting \ALCI-ontology:
  $$
    \Omc =  \{ \; A_1 \sqsubseteq \exists r^- . A_2, \ 
    A_2 \sqsubseteq \exists r^- . A_1 \; \}.
  $$
  The
  set of models of \Omc is clearly not closed under taking
  induced subinterpretations. In fact, this is true
  for every \ALCI-ontology $\Omc'$ that fits $E$ since any such $\Omc'$ must (i)~logically imply~\Omc
  and (ii)~be consistent with the ABoxes $\{ A(a)\}$ and $\{B(a)\}$, due to the negative examples.
  
   Moreover, it is  easy to see that there is no \ALC-ontology~\Omc
  that fits $E$. This is due to  Lemma~\ref{lem:forest} and the negative examples, ensuring that any such \Omc would have to be consistent with the ABoxes in $E^+$.
\end{example}


We start with a preliminary.
Let $\Lmc \in \{ \ALC, \ALCI\}$, let \Amc be an ABox, \Imc an interpretation, and
$h$ a  homomorphism from \Amc to \Imc. We define an interpretation
$\Imc_{\Amc,h, \Lmc}$ as follows. Start with interpretation $\Imc_0$:
$$
\begin{array}{rcl}
    \Delta^{\Imc_0} &=& \mn{ind}(\Amc) \\[1mm]
    A^{\Imc_0} &=& \{ a \mid h(a) \in A^{\Imc} \} \quad \text{for all}\ A \in \NC\\[1mm]
    r^{\Imc_0} &=& \{ (a,b) \mid r(a,b) \in \Amc \} \quad \text{for all}\ r \in \NR.
\end{array}
$$
Then   $\Imc_{\Amc,h, \Lmc}$ is obtained
by taking, for every $a \in \mn{ind}(\Amc)$, the $\Lmc$-unraveling of \Imc at $h(a)$ and disjointly adding
it to $\Imc_0$, identifying the root with $a$.  It can be shown that if 
\Imc is a model of some \Lmc-ontology~\Omc, then $\Imc_{\Amc,h, \Lmc}$ is also a model of \Omc.
Informally, we use $\Imc_{\Amc,h, \Lmc}$ to `undo' the potential
identification of individual names by $h$, in this way obtaining a forest model of \Amc. 
\begin{restatable}{theorem}{charUCQs}\label{thm:charUCQs}
Let $E = (E^+, E^-)$ be a collection of labeled ABox-UCQ examples with $E^-
\neq \emptyset$ and let $\Lmc \in \{ \ALC,\ALCI\}$.
Then the following are equivalent:
\begin{enumerate}
    \item there is an $\Lmc$-ontology $\Omc$ that fits $E$;
    
    \item there is an interpretation  $\Imc$ with degree at most $\bound{E}$ such that
       \begin{enumerate} 
       
       \item  $\Imc= \biguplus_{e \in E^-} \Imc_e$ where, for each $e=(\Amc,q) \in E^-$, $\Imc_e$~is an \Lmc-forest model of \Amc with $\Imc_e \not \models q$;

       \item for all $(\Amc,q) \in E^+$: if $h$ is a homomorphism from \Amc to~$\Imc$,
        then $\Imc_{\Amc,h, \Lmc} \models q$.
    
    \end{enumerate}
   
\end{enumerate}
\end{restatable}


The proof of Theorem~\ref{thm:charUCQs} follows
the same intuitions as the proofs of our previous characterizations, but is more technical. One challenge is that, in the ``$2\Rightarrow 1$'' direction, we first need to construct
from an interpretation $\Imc$ as in the theorem a suitable finite interpretation that we can then use to identify a fitting $\Lmc$-ontology. For
this we adopt the finite model construction for ontology-mediated querying from \cite{DBLP:conf/kr/GogaczIM18}.

\subsection{Upper Bounds}
\label{subsect:upperboundsCQ}

Our aim is to prove the following.
\begin{theorem}
  Let $\Lmc \in \{\ALC,\ALCI\}$ and $\Qmc \in \{ \text{CQ}, \text{UCQ} \}$. Then $(\Lmc,\Qmc)$-ontology fitting is in \TwoExpTime.
\end{theorem}
It suffices to prove the theorem for $\Qmc=\text{UCQ}$. We prove it for \ALC and \ALCI simultaneously. We use the characterization provided by Theorem~\ref{thm:charUCQs} combined with a mosaic procedure, that is, we attempt to assemble the interpretation~\Imc from Point~2
of Theorem~\ref{thm:charUCQs} by
combining small pieces.

Let $\Lmc \in \{ \ALC, \ALCI\}$ and assume
that we are given a set of ABox-UCQ examples
$E_0=(E_0^+,E_0^-)$.
We will often consider  maximally connected components of ABoxes and
CQs which, for brevity, we simply call  \emph{components}. We wish to
work with only connected queries in positive examples. This can be achieved
as follows.
If $(\Amc,q) \in E^+_0$ with $q=q_1 \vee \cdots \vee q_n$ and $q_i$ has components
$p_1,\dots,p_k$, $k> 1$, then we replace $(\Amc,q)$ with positive
examples $(\Amc,\widehat q_1),\ldots,(\Amc,\widehat q_k)$ where $\widehat q_j$ is
obtained from $q$ by
replacing the disjunct $q_i$ with~$p_j$. 
This leads to an exponential blowup
of the number of positive examples, which,
however, does not compromise our upper
bound because
the size of the examples themselves does not increase.


Throughout this section, we shall be concerned with $\Lmc$-forest models \Imc of
ABoxes $\Amc$. We generally assume the following naming convention in such models. All elements
of $\Delta^\Imc$ must be of the form $aw$ 
where $a \in \mn{ind}(\Amc)$ and $w \in \mathbb{N}^*$, that is, $w$ is a finite
word over the infinite alphabet~$\mathbb{N}$. Moreover, $(d,e) \in r^\Imc$
implies that $d,e \in \mn{ind}(\Amc)$ or $e=dc$ or $d=ec$ (if $\Lmc = \ALCI$) where $c \in
\mathbb{N}$. If $d,e \in \Delta^\Imc$ and $e=dc$,
then we call $e$ a \emph{successor} of $d$. 
Note that a successor may be connected to its predecessor via a role name, an
inverse role, or not connected at all. The \emph{depth} of $aw$ is
defined as the
length of $w$ .

Since mosaics represent `local' pieces of an interpretation,
disconnected ABoxes in examples pose a challenge: a homomorphism may
map their components into different parts of a forest model that are
far away from each other. We thus
need some preparation to deal with disconnected ABoxes. For positive
examples, one important ingredient is the following observation.
\begin{restatable}{lemma}{lemIAhcomponent} \label{lem:IAh-component}
Let $\Imc$ be an interpretation and $(\Amc, q) \in E^+$ a positive example such that Condition~(b) from Theorem~\ref{thm:charUCQs} is satisfied and each CQ in $q$ is connected.
Then there exists a component $\Bmc$ of $\Amc$ such that:
if $h$ is a homomorphism from $\Amc$ to $\Imc$, then
$\Imc_{\Bmc, h, \Lmc} \models q$
\end{restatable}
Note that Lemma~\ref{lem:IAh-component} requires the component \Bmc to
be uniform across all homomorphisms $h$. For each $e=(\Amc,q) \in E^+$, we choose
a component $\mn{ch}(e)$ of $\Amc$. Intuitively, $\mn{ch}(e)$ is the component
\Bmc from Lemma~\ref{lem:IAh-component} with 
\Imc the interpretation from  Point~2
of Theorem~\ref{thm:charUCQs}.
Since, however, we do not know~\Imc, \mn{ch} acts like a guess and our algorithm shall iterate over all possible choice functions
\mn{ch}.

To deal with an example $e=(\Amc,q) \in E^+$, we shall focus on the
component $\mn{ch}(e)$ of \Amc. The other components of \Amc, however, cannot
be ignored.  We need to know whether they have a homomorphism to \Imc,
possibly some remote part of it. This is not easily possible from the
local
perspective of a mosaic, so we again
resort to guessing. We choose a set $\Amf$ of ABoxes that are 
a component of the ABox of some positive example.
We will take care (locally!) that no ABox in \Amf admits a homomorphism
to \Imc. All other components of ABoxes in positive examples may or
may not have a homomorphism to \Imc, we shall simply treat them as if
they do.  We say that a positive example $(\Amc, q) \in E^+$ is
\emph{$\Amf$-enabled} if no component of $\Amc$ is in~$\Amf$.

Note that the number of choices for
$\Amf$ and \mn{ch} is double exponential (it is
not single exponential since  we have an exponential number of positive examples, see above). 

The queries in negative examples need not be connected. To falsify a
non-connected CQ, it clearly suffices to falsify one of its components. We use
another choice function to choose these components: for each $e = (\Amc, p_1
\lor \cdots \lor p_k) \in E^-$ and $1 \leq i \leq k$, choose a component
$\mn{ch}(e,p_i)$ of $p_i$. There are single
exponentially many choices.

Our mosaic procedure tries to assemble the $\Lmc$-forest
model \Imc starting from a large piece
that contains the ABox part of \Imc as well
as the tree parts up to depth $3 ||E||$. The potentially infinite remainder of the trees is
then assembled from smaller pieces. We start with defining the large pieces.
\begin{definition}
\label{def:basecand}
A \emph{base candidate} for \mn{ch} and $\Amf$ is an interpretation $\Jmc =
\biguplus_{e \in E^-} \Imc_e$ that satisfies the following
conditions:

\begin{enumerate}

    \item for each $e=(\Amc, p_1 \lor \cdots \lor p_k) \in E^-$, $\Imc_e$ is an
    $\Lmc$-forest model of $\Amc$ such that  $\Imc_e \not \models
    \mn{ch}(e,p_i)$ for $1 \leq i \leq k$; 
    
    \item no ABox from \Amf has a homomorphism to \Jmc;
                
    \item $\Jmc$ has depth at most $3 ||E_0||$  and degree at most $\bound{E_0}$;
  
    \item for all $e = (\Amc,q) \in E^+$ that are $\Amf$-enabled: if $h$ is a
    homomorphism from $\mn{ch}(e)$ to $\Jmc$ whose range contains only elements of depth at most $2 ||E_0||$, then \mbox{$\Jmc_{\mn{ch}(e), h, \Lmc} \models q$}.
\end{enumerate} 
\end{definition}

To make sure that there are only finitely many (in fact double exponentially many) base candidates, we
assume that (i)~\Jmc interprets only concept and role names that occur in $E$ and (ii)~if $w \in \Delta^{\Jmc}$ has $k$ successors $wc_1,\dots,wc_k$, then $\{c_1,\dots,c_k\} = \{1,\dots,k\}$.


We next define the small mosaics.
An \emph{$\Lmc$-tree interpretation} \Jmc is defined
exactly like an $\Lmc$-forest model, except that all
domain elements are of the form $w \in \mathbb{N}^*$,
that is, there is no leading individual name. We
additionally require the domain $\Delta^\Jmc$ to
be prefix-closed and call $\varepsilon \in \Delta^\Jmc$
the \emph{root} of~\Jmc. 

We say that $\Jmc'$ is a \emph{subtree} of a tree
interpretation $\Jmc$ if, for some successor $c$
of $\varepsilon$, $\Jmc'$ is the restriction of
\Jmc to all domain elements of the form $cw$, with $w \in \mathbb{N}^*$.

\begin{definition}
\label{def:mosaic}
  A \emph{mosaic for} \mn{ch}, \Amf, and $e = (\Amc, p_1 \lor \cdots \lor p_k) \in E^-$ is an $\Lmc$-tree interpretation \Mmc  that satisfies the
  following conditions:
    \begin{enumerate}

    \item $\Mmc \not \models \mn{ch}(e,p_i)$ for $1 \leq i \leq k$;

    \item no ABox from \Amf has a homomorphism to \Mmc;
    
    \item $\Mmc$ has depth at most $3 ||E_0||$  and degree at most $\bound{E_0}$; 

        \item for all $e = (\Amc,q) \in E^+$ that are $\Amf$-enabled: if $h$ is a
    homomorphism from $\mn{ch}(e)$ to $\Mmc$ whose range contains only elements of depth at least $||E_0||$ and at most $2 ||E_0||$, then $\Mmc_{\mn{ch}(e), h, \Lmc} \models q$. 
    \end{enumerate}
\end{definition}

Let  $\Imc$ be
an $\Lmc$-forest model and $d \in \Delta^{\Imc}$. With $\Imc|^\downarrow_d$,
we mean the restriction of $\Imc$ to all
elements of the form $dw$, with $w \in \mathbb{N}^*$.
We say that a mosaic $\Mmc$ \emph{glues to $d$ in} $\Imc$ if
$\Imc|^\downarrow_d$ is identical to the interpretation obtained from $\Mmc$ in the following way:
\begin{itemize}

    \item remove all elements of depth exactly $3||E_0||$;

    \item prefix every domain element with $d$, that is, every $w \in
    \Delta^\Mmc$ is renamed to $dw$.
    
\end{itemize}
Our algorithm now works as follows. In an outer
loop, we iterate over all possible choices for $\mn{ch}$ and $\mathfrak{A}$.
For each  $\mn{ch}$ and $\mathfrak{A}$, as well as for each $e \in E^-$ we construct
the set $S_{e,0}$ of all mosaics for $\mn{ch}$, $\Amf$, and $e$ and then 
apply an elimination procedure, producing a sequence of sets
$$S_{e,0} \supseteq S_{e,1} \supseteq S_{e,2} \supseteq \cdots.$$ 
More precisely, $S_{e,i+1}$ is the subset of mosaics $\Mmc \in S_{e,i}$ that satisfy
the following condition:
\begin{description}

  \item[($*$)] for all successors $c$ of $\varepsilon$, there is an $\Mmc' \in S_{e,i}$ that glues to $c$ in $\Mmc$.
  
\end{description}
Let $S_{e}$ be the set of mosaics obtained  after 
stabilization.

We next iterate over all base candidates
$\Jmc = \bigcup_{e \in E^-} \Imc_e$ for $\mn{ch}$ and $\Amc$, for each of them checking whether
there is, for every element $d \in \Delta^{\Imc_e}$ of depth~1, a mosaic $\Mmc \in S_e$ 
that glues to $d$ in $\Delta^{\Imc_e}$. If the
check succeeds for some $\mn{ch}$, $\Amf$ and~$\Jmc$, we 
return `fitting exists'. Otherwise, we return
`no fitting exists'.
\begin{restatable}{lemma}{ucqmosaic}\label{lem:ucq-mosaic}
    The algorithm returns `fitting exists' if and only if  there is an
  \Lmc-ontology that fits~$E_0$.
\end{restatable}
It remains to verify that the algorithm runs in double exponential time. Most importantly, we need an effective way
to check Condition~4 of Definitions~\ref{def:basecand} and~\ref{def:mosaic}. This is provided by the subsequent lemma.

Let $\Amc$ be an ABox and $p$ a CQ.
An \emph{\Amc-variation} of
$p$ is a CQ $p'$ that can be obtained from $p$ by consistently replacing zero or more variables with individual names from $\mn{ind}(\Amc)$ and possibly identifying  variables. We say that $p'$
is \emph{proper}
if the following conditions are 
satisfied:
\begin{enumerate}

\item if $r(a,b) \in p'$ with $a,b \in \NI$, then  $r(a,b) \in \Amc$;
  
\item 
$\Imc_{p'}$ is an $\Lmc$-forest model of $\Amc \cap p'$.



\end{enumerate}
Further, let $\Imc$ be an interpretation, $h$ a homomorphism from $\Amc$ to $\Imc$,
$p'$ an $\Amc$-variation of $p$ and $g$ a weak homomorphism from $p'$ to $\Imc$.
We say that $g$ is
\emph{compatible} with $h$ if 
 \begin{enumerate}

\item $h(a) = g(a)$ for all individual names $a$ in $p'$;

\item for every variable $x$ in $p'$, there is an $a \in \mn{ind}(\Amc)$ 
such that $g(x)$ is $\Lmc$-reachable from $h(a)$ in \Imc.

\end{enumerate}
Here, an element $e \in \Delta^\Imc$ is
\ALCI-reachable from $d \in \Delta^\Imc$
if there are $d_0,\dots,d_n \in
\Delta^\Imc$ such that $d=d_0$, $d_n=e$, and, for $0 \leq i < n$,
$(d_i,d_{i+1}) \in r^\Imc$ for some
role~$r$. In \ALC-reachability, `role' is replaced by `role name'.
\begin{restatable}{lemma}{lemvariation} \label{lem:variation}
 Let $(\Amc,q)$ be an example, \Imc an interpretation, and $h$ a homomorphism from \Amc to \Imc.
 Then the following are equivalent
 \begin{enumerate}
    \item $\Imc_{\Amc, h, \Lmc} \models q$
    \item there exists a proper \Amc-variation $p'$ of a CQ $p$ in $q$ and a weak homomorphism from $p'$ to \Imc that is compatible with~$h$.
 \end{enumerate}
\end{restatable}
The Conditions in Point~2 of Lemma~\ref{lem:variation} can clearly be
checked by brute force in single exponential time, and so can Conditions~1 to 3 of Definitions~\ref{def:basecand} and~\ref{def:mosaic}. Based on Conditions~3 in these definitions, it is thus easy to see that we can produce the set of all base candidates and of all mosaics by a straightforward enumeration in double exponential time. The elimination phase of the algorithm also
clearly needs only double exponential time.

\subsection{Lower Bounds}
\label{sect:lowerbounds}

\begin{theorem}
\label{thm:alcicq-2exphard}
    The $(\ALCI, \mathrm{CQ})$-ontology fitting problem is \TwoExpTime-hard.
\end{theorem}

To prove Theorem~\ref{thm:alcicq-2exphard},
we give a polynomial
time reduction from the complement of Boolean CQ entailment in $\ALCI$, which is 
\TwoExpTime-hard~\cite{DBLP:conf/dlog/Lutz07}.
Assume that we are given \Amc, \Omc, and $q$, and want to decide whether $\Amc \cup \Omc \models q$. We
construct a collection of labeled   ABox-CQ examples $(E^+, E^-)$ such
that $\Amc \cup \Omc \not \models q$ if and only if there is an $\ALCI$-ontology $\Omc'$ that fits $(E^+, E^-)$.

To keep the reduction simple, we assume $\Omc$ to be in \emph{normal form},
meaning that every concept inclusion in $\Omc$ has one of the following forms:
$\top \sqsubseteq A$, $A_1 \sqcap A_2 \sqsubseteq A$, $A \sqsubseteq \exists
r.B$, $\exists r.B \sqsubseteq A$, $A \sqsubseteq \neg B$, $\neg B \sqsubseteq
A$. It is well-known and easy to see that any \ALCI-ontology \Omc can be rewritten
into an ontology $\Omc'$ of this form in polynomial time, introducing fresh
concept names as needed, such that  $\Amc \cup \Omc \models q'$ if and only if
$\Amc \cup \Omc'  \models q'$ for all Boolean CQs $q'$ that do not use the
fresh concept names.

The reduction uses  fresh concept names $\mn{Real}$, $\mn{Choice}$, $\overline{\mn{Choice}}$, and~$F$, a fresh
concept name $\overline A$ for every concept name $A$ in $\Omc$ and $q$, and a fresh
role name $s$. 
It is  helpful to have the characterization in Theorem~\ref{thm:charUCQs} in
mind when reading on.

We use a single negative example to ensure that the interpretation~\Imc from
Point~2 of Theorem~\ref{thm:charUCQs} is
a model of \Amc and makes the  concept name $F$  false everywhere:
$$
(\Amc \cup \{ \mn{Real}(a) \mid a \in \mn{ind}(\Amc) \},
\exists x \, F(x)).
$$
We will use a gadget that introduces auxiliary domain elements.
To distinguish the domain elements of
primary interest from the auxiliary ones, we label
the former with the concept name \mn{Real}.

The interesting part of the reduction is to guarantee that at each element labeled with $\mn{Real}$ and for each concept name $A$ in \Omc or $q$, exactly one of the concept
names $A$ and $\overline A$ is true. It is
easy to express that both concept names cannot be true simultaneously, via the following positive example: 
$$(\{A(a), \overline A(a)\}, \exists x \, F(x)).$$
To ensure that at least one of  $A$ and $\overline A$ is true, we use the announced gadget. We first introduce
one  successor that satisfies
\mn{Choice} and one that satisfies $\overline{ \mn{Choice}}$, via the following positive examples:
$$
\begin{array}{l}
    (\{\mn{Real}(a)\}, q) \text{ with } q= \exists x\, (s(a,x)\land \mn{Choice}) \\[1mm] 
    (\{\mn{Real}(a)\}, q) \text{ with }
    q= \exists x\, (s(a,x)\land \overline{\mn{Choice}}).
\end{array}
$$
%
We then use a
positive example
$
  (\Amc^*,q^*)
$
where $\Amc^*$ and $q^*$ are displayed
in Figure~\ref{fig:cruicalexample}.
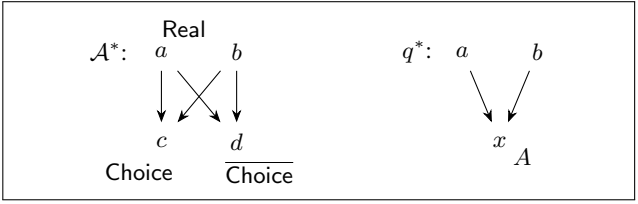
\begin{figure}
   \begin{boxedminipage}{\columnwidth}
    \centering
     \begin{tikzpicture}[
  node distance=1.0cm and 1.0cm,
  every node/.style={font=\small,  minimum size=0.5cm},
  ->, >=Stealth
]
\node (A) at (-0.7, 0) {$\Amc^*$:};
\node (a) at (0,0) {$a$};
\node (L) at (0.3,0.3) {$\mn{Real}$};
\node (b) at (1,0) {$b$};
\node (c) at (0,-1.2) {$c$};
\node (d) at (1,-1.2) {$d$};
\node (T) at (-0.3,-1.6) {$\mn{Choice}$};
\node (Tbar) at (1.3,-1.6) {$\overline{\mn{Choice}}$};

\draw (a) -- (c);
\draw (a) -- (d);
\draw (b) -- (c);
\draw (b) -- (d);

\node (q) at (3.4, 0) {$q^*$:};
\node (a) at (4,0) {$a$};
\node (b) at (5,0) {$b$};
\node (x) at (4.5,-1.2) {$x$};
\node (t) at (4.8,-1.4 ) {$A$};

\draw (a) -- (x);
\draw (b) -- (x);
\end{tikzpicture}
  \end{boxedminipage}
    \caption{Arrows denote $s$-edges.}
    \label{fig:cruicalexample}
    \vspace*{-4mm}
\end{figure}
To understand the gadget, recall Condition~(b)
of Theorem~\ref{thm:charUCQs} and the fact that
$\Imc_{\Amc^*,h,\Lmc}$ is a forest model of $\Amc^*$,
for any~\Imc. The variable $x$ in $q^*$ has two distinct individual names as a predecessor, but the only elements in $\Imc_{\Amc^*,h,\Lmc}$ with this property are
from $\mn{ind}(\Amc^*)$. It follows
that any homomorphism that witnesses $\Imc_{\Amc^*,h,\Lmc} \models q^*$ as required
by Condition~(b)
of Theorem~\ref{thm:charUCQs} maps $x$ to $c$
or to~$d$. 
We transfer the choice back to the original element: 
$$
\begin{array}{l}
(\Amc, A(a)) \text{ with } \Amc = \{\mn{Real}(a), s(a,b), \mn{Choice}(b), A(b)\} \\[1mm]
(\Amc, \overline{A}(a)) \text{ with }
\Amc = \{\mn{Real}(a), s(a,b), \overline{\mn{Choice}}(b),  A(b)\}.
\end{array}
$$
We next include positive examples that encode  $\Omc$:
$$
\begin{array}{ll}
    ( \{ \mn{Real}(a) \}, A(a)) &
    \text{for every } \top \sqsubseteq A \in \Omc \\[1mm]
    ( \{ A_1(a), A_2(a) \}, A(a)) & \text{for every }
    A_1 \sqcap A_2 \sqsubseteq A \in \Omc \\[1mm]
    \multicolumn{2}{l}{( \{ A(a)\}, q) \text{ with } %
    
    q= \exists y \, (r(a, y) \land \mn{Real}(y) \land B(y)) )} \\[1mm]
    & \text{for every } A \sqsubseteq\exists r.B \in \Omc \\[1mm] 
    ( \{ r(a, b), B(b) \}, A(a)) & \text{for every } \exists r.B \sqsubseteq A \in \Omc \\[1mm]
    ( \{ A(a) \}, \overline B(a)) & \text{for every } A \sqsubseteq \neg B \in \Omc \\[1mm]
    ( \{ \overline B(a) \}, A(a)) & \text{for every } \neg B \sqsubseteq A \in \Omc.
\end{array}
$$
Finally, we add a positive example that ensures that $F$ is non-empty if $q$ is made true:
$$( \Amc_q, \exists x \, F(x)),$$
where $\Amc_q$ is $q$ viewed as an ABox, that is, variables become individuals and atoms become assertions.
It remains to show the following.
\begin{restatable}{lemma}{lemfirstlowercorr}  $\Amc \cup \Omc \not \models q$ if and only if there is an $\ALCI$-ontology that fits $(E^+, E^-)$.
\end{restatable}
%





The reduction used in the proof of Theorem~\ref{thm:alcicq-2exphard} also
works for \ALC. But since CQ entailment
in \ALC is only \ExpTime-complete, this does not deliver the desired lower bound. We thus resort to a reduction of
the word problem for  exponentially  space-bounded alternating Turing machines (ATMs). Such reductions have been used oftentimes for DL query entailment problems, see e.g.\ \cite{DBLP:conf/dlog/Lutz07,DBLP:conf/ijcai/EiterLOS09,DBLP:conf/aaai/BednarczykR22}.
%
%
%
%
%
%
%
%

\begin{restatable}{theorem}{ALCCQTwoExpHard}
\label{thm:ALCCQTwoExpHard}
The (\ALC, CQ)-ontology fitting problem is \TwoExpTime-hard. 
\end{restatable}
The crucial step in an ATM reduction of this kind is  to ensure that
tape cells of two consecutive configurations are labeled
in a matching way. This is typically achieved by copying the labeling of each configuration to all successor configurations so
that the actual comparison can 
take place locally.
We achieve this with a
gadget that is based on the same basic idea as the gadget used in the
proof of Theorem~\ref{thm:alcicq-2exphard}, but much more intricate. 

\section{Conclusion}
\label{sect:concl}

We introduced ontology fitting problems based on ABox-query examples and presented algorithms and
complexity results, concentrating on
the ontology languages \ALC and \ALCI.
We believe that our results can be 
adapted to cover many common extensions
of these. As an illustration, we show in the appendix the following result for
the extension \ALCQ of \ALC with qualified number restrictions. A homomorphism $h$ from an ABox $\Amc_1$ to an ABox $\Amc_2$ is
\emph{locally injective} if $h(b) \neq h(c)$ for all $r(a,b),r(a,c) \in \Amc_1$.
\begin{restatable}{theorem}{thmALCQ}
	Let $E=(E^+,E^-)$ be a collection of labeled ABox examples 
	and $\Amc^+ = \biguplus E^+$. Then  the following 
    are equivalent:
    \begin{enumerate}

        \item     $E$ admits a fitting
    \ALCQ-ontology;

    \item
        there is no homomorphism from any $\Amc \in E^-$ to $\Amc^+$
	that is locally injective.
    \end{enumerate}

\end{restatable}   
Apart from extensions of \ALC,
there are many other natural 
 ontology languages of interest that
 can be studied in future work, including Horn DLs such as 
\EL and existential rules. One can also
vary the framework in several natural 
ways and, for instance, consider the case where a signature for
the fitting ontology is given as an 
additional input or where
negative examples 
have a stronger semantics, namely
$\Amc \cup \Omc \models \neg q$ in place
of $\Amc \cup \Omc \not\models q$.

\paragraph{Acknowledgements.} The third author was supported by DFG project LU 1417/4-1.









    


\cleardoublepage

\bibliographystyle{kr}
\bibliography{local}

\appendix

\cleardoublepage

\section{Proofs for Section~\ref{sect:consfit}}

To prove  Theorem~\ref{thm:conschar}, we make use of the connection between ontology-mediated
querying and constraint satisfaction problems (CSPs) established in 
\cite{DBLP:conf/kr/LutzW12}. Theorems 20 and 22 of that paper state the following.
A \emph{signature} is a set $\Sigma$ of concept
and role names. For any syntactic object $O$ such as an ontology, an ABox, or a collection of examples, we use $\mn{sig}(O)$, we
denote the set of concept and role names used in $O$. With a \emph{$\Sigma$-ABox}, we mean an ABox \Amc with $\mn{sig}(\Amc) \subseteq \Sigma$.
\begin{proposition}
\label{prop:fromKR12}~\\[-4mm] 
  \begin{enumerate}
  
    \item For every \ALCI-ontology \Omc and finite signature $\Sigma \supseteq \mn{sig}(\Omc)$, there is a $\Sigma$-ABox $\Amc_\Omc$  such that 
      for all $\Sigma$-ABoxes~\Amc: \Amc is consistent with \Omc if and only if \mbox{$\Amc \rightarrow \Amc_\Omc$};

     \item For every ABox \Amc and finite signature $\Sigma \supseteq \mn{sig}(\Amc)$, there is an \ALC-ontology $\Omc_{\Amc,\Sigma}$ such that
       for all $\Sigma$-ABoxes \Bmc: $\Bmc \rightarrow \Amc$ if and only if \Bmc is consistent with~$\Omc_{\Amc,\Sigma}$.
     
  \end{enumerate}
\end{proposition}
For the reader's information, we recall the construction of the ontology $\Omc_{\Amc, \Sigma}$ from Point~2. This is
of interest because, in the proof of Theorem~\ref{thm:conschar},
that ontology will be use as the fitting ontology (if there is one). We introduce a fresh concept name $V_a$ for every $a \in \mn{ind}(\Amc)$ and then define $ \Omc_{\Amc,\Sigma}$ to contain the following concept inclusions.
$$
\begin{array}{r@{\;}c@{\;}ll}
 \top &\sqsubseteq& \multicolumn{2}{l}{\displaystyle\bigsqcup_{a \in \mn{ind}(\Amc)} V_a} \\[5mm]
                 V_a \sqcap V_b &\sqsubseteq& \bot &\    a, b \in \mn{ind}(\Amc)  \text { with }  a \neq b \\[1mm]
                 V_a &\sqsubseteq& \neg A & \ a \in \mn{ind}(\Amc), A \in \Sigma, A(a) \notin \Amc\\[1mm]
                 V_a &\sqsubseteq& \forall r. \neg V_b  &\ a, b \in \mn{ind}(\Amc), r \in \Sigma, r(a, b) \notin \Amc.
\end{array}
$$
By construction of $\Omc_{\Amc, \Sigma}$, a common model \Imc of a $\Sigma$-ABox \Bmc and $\Omc_{\Amc, \Sigma}$ gives rise to a homomorphism $h$ from \Bmc to \Amc by setting $h(b) = a$ if $b \in V_a$ for all $a \in \mn{ind}(\Bmc)$. This  is well-defined since \Imc is a model of $\Omc_{\Amc, \Sigma}$, and thus the sets $V_a^\Imc$ form a partition of $\Delta^\Imc$.
\thmconschar*
\begin{proof}
``$1 \Rightarrow 2$''. Let \Omc be an \ALCI-ontology that fits~$E$.
Assume, to the contrary of what we have to show, that $\Amc \rightarrow \Amc^+$ for some $\Amc \in E^-$.  Since
\Omc fits $E$, every $\Amc \in E^+$ is consistent with \Omc. By taking the 
disjoint union of models which witness this, we obtain a model of $\Amc^+$ and \Omc and thus $\Amc^+$ is consistent with \Omc. Therefore,
$\Amc^+ \rightarrow \Amc_\Omc$ where $\Amc_{\Omc}$ is the $\Sigma$-ABox from Point~1 of Proposition~\ref{prop:fromKR12},
for $\Sigma=\mn{sig}(O) \cup \mn{sig}(E)$.
By composition of homomorphisms, we obtain $\Amc \rightarrow \Amc_{\Omc}$. By choice of
$\Amc_{\Omc}$, this implies that $\Amc$ is
consistent with \Omc. This contradicts the fact that \Omc fits $E$.
%

\smallskip

``$2 \Rightarrow 1$''. Assume that  $\mathcal{A} \not\rightarrow \Amc^{+}$ for all $\Amc \in E^-$. 
We argue that the \ALC-ontology $\Omc_{\Amc^+,\Sigma}$ from Point~2 of Theorem~\ref{prop:fromKR12} fits $E$,
for $\Sigma=\mn{sig}(E)$. First,
let $\Amc \in E^+$ be a positive example.  Clearly, $\Amc \rightarrow \Amc^+$. Thus, \Amc is consistent with $\Omc_{\Amc^+,\Sigma}$ by choice of $\Omc_{\Amc^+,\Sigma}$. Now let $\Amc \in E^-$ be a negative example.  
By assumption,  $\mathcal{A} \not\rightarrow \Amc^{+}$ and thus $\Amc$ is not consistent with $\Omc_{\Amc^+,\Sigma}$.
\end{proof}

\thmconsnPcompl*
\begin{proof}
    Theorem~\ref{thm:conschar} places the complement of the consistent \Lmc-ontology fitting problem
    in \NPclass: construct $\Amc^+$ in
    polynomial time, guess an $\Amc \in E^-$ and 
    a homomorphism  from 
 $\Amc$ to $\Amc^+$. The lower bound is by a straightforward reduction of the homomorphism problem between directed graphs $G_1$ and $G_2$. Simply use $G_1$, viewed as an ABox in the obvious way, as the only negative 
 example and $G_2$ as the only positive example; then invoke Theorem~\ref{thm:conschar}.
\end{proof}

\section{Proofs for Section~\ref{sect:aqs}}

We again use a result that connects fitting to CSPs, in the
style of Proposition~\ref{prop:fromKR12}, but for AQs in place
of consistency~\cite{DBLP:conf/kr/BourhisL16}.

\begin{proposition}\label{prop:fullhom} \mbox{} 
  \begin{enumerate}
  
    \item For every \ALCI-ontology \Omc and finite signature $\Sigma \supseteq \mn{sig}(\Omc)$, there is a $\Sigma$-ABox $\Amc_\Omc$  such that 
      for all $\Sigma$-ABoxes \Amc, $a \in \mn{ind}(\Amc)$, and AQs $Q$: $\Amc \cup \Omc \not\models Q(a)$
      if and only if there is a homomorphism $h$ from \Amc to $\Amc_\Omc$ with $Q(h(a)) \notin \Amc_\Omc$;


        \item For every ABox \Amc and finite signature $\Sigma \supseteq \mn{sig}(\Amc)$,
    there is an \ALC-ontology $\Omc_{\Amc,\Sigma}$ such that
       for all $\Sigma$-ABoxes \Bmc, $b \in \mn{ind}(\Bmc)$, and concept names $Q$: there is a homomorphism $h$ from \Bmc to \Amc 
       with $Q(h(b)) \notin \Amc$ if and only if $\Bmc \cup \Omc_{\Amc,\Sigma} \not\models Q(b)$.
  
  \end{enumerate}
\end{proposition}

\thmAQchar*
\begin{proof}
``$1 \Rightarrow 2$''. Assume that there is an \ALCI-ontology \Omc that fits~$E$.  
Take any $(\Amc,Q(a)) \in E^-$. 
Then $\Amc \cup \Omc \not\models Q(a)$ and thus $\Amc$ and $\Omc$ have a common model \Imc with $a \notin Q^\Imc$.
By taking the disjoint union of these
models for all $(\Amc,Q(a)) \in E^-$, we obtain a  model \Imc of $\Amc^-$ and \Omc such that  $a \notin Q^\Imc$ for all  $(\Amc,Q(a)) \in E^-$. 
We define an ABox \Cmc as follows:
\begin{align*}
    \Cmc : = {} & \Amc^- \cup {} \\
    & \{ Q(b) \mid (\Amc, Q(a)) \in E^+, b \in \mn{ind}(\Amc^-), b \in Q^{\Imc}\}.
\end{align*}
  It is straightforward to verify that \Cmc is a completion for $E$ that satisfies Condition~(b). It remains to show that it satisfies also 
  Condition~(a).
  
  Let $(\Amc,Q(a)) \in E^{+}$ and let $h$ be a homomorphism from
$\Amc$ to \Cmc. Assume to the contrary of what we have to show that $Q(h(a)) \notin \Cmc$. By definition  of \Cmc, this implies 
$h(a) \notin Q^\Imc$. Thus \Imc witnesses that $\Cmc \cup \Omc \not \models Q(h(a))$.
Let $\Amc_{\Omc}$ be the $\Sigma$-ABox from Point~1 of Proposition~\ref{prop:fullhom}, where $\Sigma = \mn{sig}(E)\cup \mn{sig}(\Omc)$. Since
$\Cmc \cup \Omc \not \models Q(h(a))$,  by choice of  $\Amc_{\Omc}$ there is a homomorphism
  $g$ from \Cmc to $\Amc_{\Omc}$ with $Q(g(h(a))\notin \Amc_{\Omc}$. Then $f=g \circ h$ is a homomorphism from  $\Amc$ to $\Amc_{\Omc}$
  such that $Q(f(a)) \notin \Amc_{\Omc}$. By choice of  $\Amc_{\Omc}$,
  this implies 
    $\Amc \cup \Omc \not\models Q(a)$, in contradiction to the fact that \Omc fits $E$.

    \smallskip

``$2 \Rightarrow 1$''. Let \Cmc be a completion for $E$ that satisfies Conditions~(a) and~(b).
Consider the ontology $\Omc_{\Cmc,\Sigma}$ from Point 2 of 
Proposition~\ref{prop:fullhom}, for $\Sigma \supseteq \mn{sig}(E)$. We argue that $\Omc_{\Cmc}$ fits~$E$. First let $(\Amc,Q(a)) \in E^+$.
Then by Condition~(a) every homomorphism $h$ from \Amc to \Cmc satisfies \mbox{$Q(h(a)) \in \Cmc$}. By 
choice of $\Omc_{\Cmc,\Sigma}$, this gives $\Amc \cup \Omc_{\Cmc,\Sigma} \models Q(a)$, as required.
Now let $(\Amc,Q(a)) \in E^-$. By Condition~(b), we have
$Q(a) \notin \Cmc$.  The identity may thus serve as a homomorphism $h$ from $\Amc$ to~\Cmc with  $Q(h(a)) \notin \Cmc$. By choice of $\Omc_{\Cmc,\Sigma}$,
this implies $\Amc \cup \Omc_{\Cmc,\Sigma} \not \models Q(a)$, as required.
\end{proof}

\proprefcand*
\begin{proof}
    ``$1 \Rightarrow 2$''.  We prove the contraposition. Assume that there is no refutation candidate \Cmc for $E$ such that 
     $Q(a) \in \Cmc$ for some $(\Amc,Q(a)) \in E^-$.
    Let \Cmc be the refutation candidate obtained from $\Amc^-$ by 
    applying rule \Rsf exhaustively. Then   $Q(a) \notin \Cmc$ for all $(\Amc,Q(a)) \in E^-$. Thus, 
    \Cmc is a completion that satisfies Conditions~(a) and~(b) of Theorem~\ref{thm:AQchar}. Thus, Theorem~\ref{thm:AQchar} implies that $E$ admits a fitting \Lmc-ontology.

       \smallskip

``$2 \Rightarrow 1$''. Assume that there is a refutation candidate \Cmc for $E$ such that $Q(a) \in \Cmc$ for some $(\Amc,Q(a)) \in E^-$.
Then every completion \Cmc for $E$ that satisfies Condition~(a) of Theorem~\ref{thm:AQchar} must fail to satisfy Condition~(b). To see
this, it suffices to note that the rule \Rsf from the definition 
of refutation candidates is in fact identical to Condition~(a).
Thus, Point~2 of Theorem~\ref{thm:AQchar} fails, and that theorem implies that $E$ does not admit a fitting \Lmc-ontology.
\end{proof}

\section{Proofs for Section~\ref{sect:fullCQ}}
%

\lemNoInconsEx*
\begin{proof}
    `` $\Leftarrow$ ''
    Assume that \Omc is a fitting \Lmc-ontology for $E'$. We claim that \Omc is also a fitting \Lmc-ontology for $E$. Since the sets of positive examples in both collections coincide, it remains to show that $\Amc \cup \Omc \not\models q$ for every negative example $(\Amc,q) \in E$. 
    If $(\Amc, q)$ is a consistent example, then $(\Amc, q) \in E'^-$, and therefore  
    $\Amc \cup \Omc \not \models q$ by assumption. Now, let $(\Amc, q) \in E^-$ be an inconsistent example. By construction, $(\Amc, X(a)) \in E'^+$ and therefore $\Amc \cup \Omc$ must be consistent. Using the definition of inconsistent examples, we conclude $\Amc \cup \Omc 
    \not \models q$.
    
    `` $ \Rightarrow$ '' 
    Assume that \Omc is a fitting \Lmc-ontology for $E$. We construct $\Omc'$ by replacing every occurrence of $X$ in \Omc with the fresh concept name $Y \in \NC$. Next, for every interpretation \Imc, define $\Imc'$ to be the interpretation obtained from \Imc by setting $X^{\Imc'} := \emptyset$ and $Y^{\Imc'} := X^\Imc$. Observe that for every interpretation $\Jmc$ with $X^{\Jmc} = \emptyset$, there exists an interpretation $\Imc$ such that $\Imc' = \Jmc$. Using induction, one can easily show that \Imc is a model of  $\Amc \cup \Omc$ if and only if $\Imc'$ is a model of $\Amc \cup \Omc'$, for every ABox \Amc in which neither $X$ nor $Y$ appears. 
    
    Now, let $\Amc$ be an ABox and $q$ a full conjunctive query, in none of which $X$ or $Y$ occurs. 
    \\[1mm]
    {\bf Claim.} $\Amc \cup \Omc \models q$ if and only if $\Amc \cup \Omc ' \models q$. 

    For the only if direction, assume $\Amc \cup \Omc' \not \models q$ and let $\Jmc$ be the respective witness model of $\Amc \cup \Omc'$ satisfying $\Jmc \not \models q$. Since $X$ does not occur in $\Omc'$, \Amc or $q$, we may assume $X^{\Jmc} = \emptyset$. Let \Imc be the interpretation such that $\Imc' = \Jmc$. By the earlier argument, $\Imc$ is a model of $\Amc \cup \Omc$. Furthermore, as $X$ and $Y$ do not occur in $q$, we reason $\Imc \not \models q$, and hence $\Amc \cup \Omc \not \models q$. 

    To show the if direction, conversely suppose $\Amc \cup \Omc \not \models q$. Let \Imc be the witness model of $\Amc \cup \Omc$ with $\Imc \not \models q$. Then, $\Imc'$ is a model of $\Amc \cup \Omc'$ and again, since neither $X$ nor $Y$ do  appear in $q$, we infer $\Imc'\not \models q$. Thus, $\Amc \cup \Omc' \not \models q$.
    
    \medskip
    We now prove that $\Omc'$ is a fitting ontology for $E'$. Let $(\Amc, q) \in E'^+ = E^+$. Since \Omc is a fitting ontology for $E$, $\Omc \cup \Amc \models q$ and using the claim above, we derive $\Omc' \cup \Amc \models q$. An analogous argument applies to each consistent negative example $(\Amc, q) \in E^-$. It remains to show that for every inconsistent negative example $(\Amc, q) \in E^-$, we have $\Amc \cup \Omc' \not \models X(a)$.  By assumption, there exists a model \Imc of $\Amc \cup \Omc$, and thus $\Imc'$ is a model of $\Amc \cup \Omc' $. Since $X^{\Imc'}= \emptyset$, we conclude $\Imc' \not \models q$, and thus $\Omc'$ is a fitting ontology for $E'$.
\end{proof}

\thmfullCQchar*
\begin{proof}
    ``$1 \Rightarrow 2$''. Assume that $E$ admits a fitting \Lmc-ontology \Omc. 
    Define a collection $E'$ 
    of ABox-AQ examples as follows:
    \begin{itemize}

        \item ${E'}^+ = \{ (\Amc,A(a)) \mid (\Amc,q) \in E^+ \text{ and } A(a) \in q\}$;

        \item Consider each 
        $(\Amc,q) \in E^-$; since no inconsistent negative examples are used $(\Amc,q)$ is consistent. Thus every role atom $r(a,b) \in q$ is an assertion in \Amc. Since \Omc fits the negative example $(\Amc,q)$, there must thus be
           a concept assertion $Q(a) \in q$ such that $\Amc \cup \Omc \not\models Q(a)$. Include $(\Amc,Q(a))$ in~${E'}^-$.
           
    \end{itemize}
    It is easy to verify that \Omc fits $E'$. Thus, there is a completion
    \Cmc for $E'$ that satisfies Conditions~(a) and~(b) from Theorem~\ref{thm:AQchar}. From the proof of that theorem, we
     additionally know that \Cmc is consistent with \Omc. By construction of $E'$, $\Cmc$ being a completion of $E'$ clearly
    implies that \Cmc also satisfies Conditions~(a) and~(b) from Theorem~\ref{thm:fullCQchar}. It remains to argue that 
     Condition~(c) is also satisfied. 
     
     To this end, assume that $(\Amc,q)\in E^+$
     is inconsistent. Then \Amc is inconsistent with \Omc. Assume to
     the contrary of what we have to show that there is a homomorphism $h$ from \Amc to \Cmc.
     Let $\Amc_{\Omc}$ be the $\Sigma$-ABox from Point~1 of Proposition~\ref{prop:fromKR12}, for $\Sigma = \mn{sig}(\Omc) \cup \mn{sig}(E')$. Since \Cmc is consistent with \Omc, there is a homomorphism $g$ from \Cmc to $\Amc_{\Omc}$.
     By composing $h$ and $g$, we obtain a homomorphism from $\Amc$ to
     $\Amc_{\Omc}$. As this implies that \Amc is consistent with \Omc, we
     have obtained a contradiction.
     
       \smallskip

  ``$2 \Rightarrow 1$''. Assume that there is a completion \Cmc for $E$ such that Conditions~(a) to~(c) from Theorem~\ref{thm:fullCQchar} are satisfied. Define a collection $E'$ of ABox-AQ examples as follows:
    \begin{itemize}

        \item ${E'}^+ = \{ (\Amc,A(a)) \mid (\Amc,q) \in E^+ \text{ and } A(a) \in q\}$;

        \item Consider each $(\Amc,q) \in E^-$; by Condition~(b) and since all negative examples are consistent,
        there is a $Q(a) \in q$ such that $Q(a) \notin \Cmc$. Include $(\Amc,Q(a))$ in ${E'}^-$.
           
    \end{itemize}
    It is easy to see that \Cmc is also a completion for $E'$ that satisfies Conditions~(a) and~(b) from Theorem~\ref{thm:AQchar}.
    By that theorem, there is therefore an \Lmc-ontology \Omc that fits $E'$.
    We argue that \Omc also fits $E$. Since \Omc fits all negative examples in $E'$, it is clear
    the \Omc fits all negative examples of $E$. 
    
    Let $(\Amc,q) \in E^+$
    be a consistent positive example. Since \Omc
    fits all positive examples in $E'$, we have $\Amc \cup \Omc \models Q(a)$ for all concept atoms $Q(a) \in q$. It thus remains to show that 
    $\Amc \cup \Omc \models r(a,b)$ for all role atoms $r(a,b) \in q$, but
    this is clear since $(\Amc,q)$ is consistent.

    Now let $(\Amc,q) \in E^+$
    be an inconsistent positive example. By Condition~(c), there is no homomorphism from \Amc to \Cmc. From the proof of Theorem~\ref{thm:AQchar}, we actually know that \Omc can be chosen as the
    ontology $\Omc_{\Cmc, \Sigma}$ from Point~2 of Proposition~\ref{prop:fullhom}, where $\Sigma = \mn{sig}(E') \cup \{Q\}$ for a fresh concept name $Q$.
    Now, choose some $b \in \mn{ind}(\Amc)$. Since there is no homomorphism from
    \Amc to \Cmc, there is also no homomorphism $h$ from \Amc to \Cmc with
    $Q(h(b)) \notin \Cmc$. By choice of $\Omc_{\Cmc, \Sigma}$, this implies $\Amc \cup \Omc_{\Omc, \Sigma} \models Q(b)$. But since $Q$ is a fresh concept
    name, \Amc must then be inconsistent with $\Omc_{\Cmc, \Sigma}$. Thus, $\Omc=\Omc_{\Cmc, \Sigma}$ fits $(\Amc,q)$.
\end{proof}

\thmfullCQconpupper*

\begin{proof}
       Again \coNPclass-hardness is inherited from the AQ case and it remains to
   argue that the complement of the $(\Lmc,\text{FullCQ})$-ontology fitting problem is in \NPclass. By Proposition~\ref{prop:refcandFullCQ}, it suffices to guess 
  an ABox \Cmc with the same individuals as $\Amc^-$ and to verify that it is a refutation candidate and that it satisfies at least one of Conditions~(a) and~(b) in Proposition~\ref{prop:refcandFullCQ}. Verifying that \Cmc is a refutation candidate can be achieved 
  exactly as in the proof of Theorem~\ref{thm:AQincoNP}. To verify that
  \Cmc satisfies at least one of Conditions~(a) and~(b), we may guess which condition is satisfied and, in the case of Condition~(b),
  also an inconsistent $(\Amc,q) \in E^+$ and a homomorphism from \Amc
  to \Cmc. Verifying that $(\Amc,q)$ is indeed inconsistent can clearly
  be done in polynomial time.
\end{proof}

\section{Proof of Theorem~\ref{thm:charUCQs}}
To prove~Theorem~\ref{thm:charUCQs}, we need several tools. One of them is a local variant of $\ALCI$-bisimulations:

Let $\Imc_1, \Imc_2$ be interpretations and $k \geq 0$. A relation $S$ is a \emph{$k$-$\ALCI$-bisimulation} between $\Imc_1$ and $\Imc_2$ if there is a series of relations
$S = S_k \subseteq S_{k - 1} \cdots \subseteq S_0$ such that  the following
conditions are satisfied for all concept names $A$, roles $r$, and  $i \geq 1$:
\begin{enumerate}
    \item if $(d_1, d_2) \in S_0$, then $d_1 \in A^{\Imc_{1}}$ if and only if $d_2 \in A^{\Imc_{2}}$;
    \item if $(d_1, d_2) \in S_i$ and $(d_1, d_1') \in r^{\Imc_1}$, then there is a $(d_2, d_2') \in r^{\Imc_2}$ with $(d_1', d_2') \in S_{i - 1}$;
    \item if $(d_1, d_2) \in S_i$ and $(d_2, d_2') \in r^{\Imc_2}$, then there is a $(d_1, d_1') \in r^{\Imc_2}$ with $(d_1', d_2') \in S_{i - 1}$.
\end{enumerate}
We write $(\Imc_1, d_1) \sim_{\ALCI, k} (\Imc_2, d_2)$ if there is a $k$-$\ALCI$-bisimulation $S$ between interpretations $\Imc_1$ and $\Imc_2$ with $(d_1, d_2) \in S$. We also define $k$-$\ALC$-bisimulations as the variant of $k$-$\ALCI$-bisimulations where $r$ ranges only over role names and write $(\Imc_1, d_1) \sim_{\ALC, k} (\Imc_2, d_2)$ if there is a $k$-$\ALCI$-bisimulation $S$ between interpretations $\Imc_1$ and $\Imc_2$ with $(d_1, d_2) \in S$.

 The following lemma is standard: 

\begin{lemma}\label{lem:bisim-preserves-trees}
    Let $\Jmc, \Jmc'$ be  interpretations, $e \in \Delta^\Jmc$, $e' \in \Delta^{\Jmc'}$, $\Lmc \in \{\ALC, \ALCI\}$, $\Imc$ an $\Lmc$-forest
    model of $\Amc = \emptyset$ consisting of a single tree with root $d$ and depth at most $n$.
    If there is a
    homomorphism $h$ from $\Imc$ to $\Jmc$ with $h(d) = e$ and  $(\Jmc, e) \sim_{\Lmc, n} (\Jmc', e')$, then there is also a homomorphism
    $h'$ from $\Imc$ to $\Jmc'$ with \mbox{$h(d) = e'$}. 
\end{lemma}

Note that if $\Jmc'$ is the $\Lmc$-unraveling of $\Jmc$ at $d$, then
$(\Jmc, e) \sim_{\Lmc, n} (\Jmc', e)$ for any $n$.

The theorem below is a direct consequence of the model construction used in the proof of
Theorem~1 in \cite{DBLP:conf/kr/GogaczIM18}.

\begin{theorem}\label{thm:finite-counter-model}
    Let \Amc be an ABox, $\Imc$ an $\ALCI$-forest model of $\Amc$, and $n \geq 1$. Then, there
    exists a finite model $\Jmc$ of $\Amc$ such that 
        for all CQs with at most $n$ variables: for every homomorphism $h$ from $q$ to $\Jmc$, there is a
        homomorphism $h'$ from $q$ to $\Imc$ such that
        \[
            (\Imc, h(x)) \sim_{\ALCI, n} (\Jmc, h'(x))
        \]
        for all $x \in \mn{var}(q)$.
\end{theorem}

Moreover, we need to restrict the degree of interpretations used in Point 2 of Theorem~\ref{thm:charUCQs} to be bounded by an exponential and introduce a stronger version of Condition~(b) of Theorem~\ref{thm:charUCQs} that, informally, ensures that matches of disconnected or Boolean CQs are local.

For an $\Lmc$-forest models $\Imc$ of some ABox $\Amc$, let the \emph{depth} of an element $d$ be the length of
the shortest path from an individual $a \in \mn{\Amc}$ to $d$, or $\infty$ if no such path exists.
For $n \geq 0$, let $\Imc|_n$ be the restriction of $\Imc$ to elements of at most depth $n$.

\begin{lemma}\label{lem:make-IAh-local}
Let $E = (E^+, E^-)$ be a collection of labeled ABox-UCQ examples.
If there is an interpretation $\Imc$ that
satisfies Conditions~(a) and~(b) of Theorem~\ref{thm:charUCQs}, then there exists an interpretation $\Imc'$ with degree at most \bound{E} that satisfies Condition~(a), (b) and the following
variant of Condition~(b):
\begin{enumerate}
    \item[(b${}^*$)] for all $(\Amc, q) \in E^+$, if $h$ is a homomorphism from
    $\Amc$ to $\Imc'$, then $\Imc'_{\Amc, h, \Lmc}|_{||E||} \models q$.
\end{enumerate}
\end{lemma}
\begin{proof}
We modify $\Imc$ as follows to obtain $\Imc'$.
Let $u$ be a role name that does not occur in $E$, and let $n = ||E||$.

First we modify $\Imc$ to ensure that Condition~(b${}^*$) holds.
As $\Imc$ satisfies Condition~(b),  for each $(\Amc, q) \in E^+$ and homomorphism $h$ from $\Amc$ to $\Imc$,
there must be a CQ $p$ in $q$ such that $\Imc_{\Amc, h, \Lmc} \models p$.
Let $g$ be a homomorphism from $p$ to $\Imc_{\Amc, h, \Lmc}$, and $h'$ be the extension of $h$ to a homomorphism from $\Imc_{\Amc, h, \Lmc}$ to $\Imc$.
For each connected component $p'$ of $p$ that does not contain an individual from $\Amc$, pick a variable $x \in \mn{var}(p')$,
and an individual $a \in \mn{ind}(\Amc)$ and extend $\Imc$ with $(h'(a), h'(g(x)) ) \in u^{\Imc}$.

After this modification, $\Imc$ satisfies Condition~(b${}^*$), but is no longer an $\Lmc$-forest model, and has potentially infinite degree.
Then, by unraveling $\Imc$, we obtain an $\Lmc$-forest model $\Jmc$ that satisfies Condition~(a) and~(b${}^*$). 

We obtain the desired interpretation $\Imc'$ as a restriction of $\Jmc$ to a subset of its domain.
We start by setting $\Imc'$ to be the restriction of $\Jmc$ to the individuals of
$\bigcup_{(\Amc, q) \in E^-} \Amc$. Thus, at the start the degree of $\Imc'$
is bounded by $||E^-||$ and $\Imc'$ satisfies Condition~(a).

We then extend $\Imc'$ by exhaustively applying the following rule for all $(\Amc, q) \in E^+$:
\begin{enumerate}
    \item[($*$)] if $h$ is a homomorphism from $\Amc$ to $\Imc'$ and $\Imc'_{\Amc, h, \Lmc}|_n \not\models q$, then choose a CQ $p$ of $q$ such that $\Jmc_{\Amc, h, \Lmc}|_n \models p$.
    Now, from all homomorphisms from $p$ to $\Jmc_{\Amc, h, \Lmc}|_n$ select a homomorphism $g'$ for which $\mn{im}(g') \cap \Delta^{\Imc'_{\Amc, h, \Lmc}}$ is maximal.

    For each component $p'$ of $p$, let $\Jmc'$ be the restriction of $\Jmc_{\Amc, h, \Lmc}$ to the minimal set of elements that contains the image of $p'$ under $g'$ and is connected to some individual $a \in \mn{\Amc}$.
    Extend $\Imc'$ with the image of $\Jmc'$ under $h'$, where $h'$ is the natural extension of $h$ to $\Jmc_{\Amc, h, \Lmc}$.
\end{enumerate}


Exhaustive and fair application of ($*$), such that every homomorphism $h$ is eventually processed, ensures that $\Imc'$ satisfies
Condition~(b${}^*$). Since it always remains a restriction of $\Jmc$ to a subset of its domain,
it also satisfies Condition~(a).

We now show that the degree of any element $d \in \Delta^{\Imc'}$ is bounded by \bound{E}. 
For this, we make use of Lemma~\ref{lem:variation} and the definition of \Amc-variations from Section~\ref{subsect:upperboundsCQ}. 
It is easy to see that this bound holds for the initial interpretation $\Imc'$ before any rule application.
Now consider a successor $d' \in \Delta^{\Imc'}$ of $d$, introduced by a rule application for the positive example $(\Amc, q)\in E^+$ and homomorphism $h$ from \Amc to $\Imc'$. Let $p$ be the CQ of $q$ and $g'$ the homomorphism from $p$ to $\Jmc_{\Amc,h,\Lmc}|_n$ selected by the rule. By the same argument as in the proof of Lemma~\ref{lem:variation}, there exists 
a proper \Amc-variation $p'$ of $p$ along with a weak homomorphism $\hat h = h' \circ g'$ from $p'$ to $\Imc'$ that is compatible with $h$, where $h'$ denotes the extension of $h$ to $\Imc'_{\Amc, h, \Lmc}$. Since $d'$ lies in the image of $\hat h$, for each $t \in \hat{h}^{-1}(d')$, there is a subtree of $p'$ rooted at  $t$  that maps into the subtree of $\Imc'$ rooted at $d'$. Due to the choice of $g'$, no other successor of $d$ allows a mapping from all of these subtrees. Consequently, for every such subtree of an \Amc-variation of $p$, at most one distinct successor is introduced over the course of the entire algorithm.  

This restricts the number of successors introduced per positive example $(\Amc, q) \in E^+$ to
$$(\|(\Amc, q)\|+1)^{\|q\|} = 2 ^{\|q\| \cdot \log (\|(\Amc, q)\|+1)}.$$
Summing this over all positive examples and the initial interpretation yields the overall bound \bound{E} on the degree of any element in $\Delta^{\Imc'}$.
\end{proof}

\charUCQs*

\begin{proof}
 ``$1 \Rightarrow 2$''.
 Let \Omc be an \Lmc-ontology that fits $( E^{+}, E^{-})$.
 Then $ \Amc \cup \Omc \not\models q $ for all $ ( \Amc, q ) \in
 E^{-} $. By Lemma~\ref{lem:forest}, we can thus choose for every $e = (\Amc, q) \in E^-$ an $\Lmc$-forest model of $\Amc$ and $\Omc$ as $\Imc_e$ such that $\Imc_e \not\models q$ and degree at most $||\Omc||$. By construction, $\Imc:= \biguplus_{e\in E^-} \Imc_e$ satisfies Condition~(a).
 Note that $\Imc:= \biguplus_{e\in E^-} \Imc_e$
is a well-defined interpretation (i.e., it has a non-empty domain)
due to our assumption that $E^- \neq \emptyset$.

Next, we show that $\Imc$ also satisfies Condition~(b).
For this let $(\Amc,q) \in E^+$ and let $h$ be a homomorphism from \Amc to $\Imc$. Observe that \Imc is a model of \Omc since each
$\Imc_e$ is. This implies that $\Imc_{\Amc,h, \Lmc}$ is not only a model of $\Amc$, but also a model of \Omc. Since
\Omc fits $E^+$, it must be that $\Amc \cup \Omc \models q$, and thus $\Imc_{\Amc,h, \Lmc} \models q$, as required.

It then follows from Lemma~\ref{lem:make-IAh-local} that there also exists an interpretation $\Imc'$ with degree at most \bound{E} that satisfies Conditions~(a) and~(b).

\medskip

``$2 \Rightarrow 1$''. 
Let $(\Imc_e)_{e \in E^-}$ be interpretations such that $\Imc= \biguplus_{e \in
E^-} \Imc_e$ satisfies Conditions~(a) and~(b). 
Without loss of generality we assume that $\Imc$ interprets as non-empty only the concept and
role names in $E$.
We proceed as follows:
first, we convert $\Imc$ into a finite model $\Jmc$ that also satisfies
Conditions~(a) and~(b). Then, we use $\Jmc$ to construct the desired ontology
$\Omc$ that fits $(E^+, E^-)$.
By Lemma~\ref{lem:make-IAh-local}, we can assume that $\Imc$ satisfies
Condition~(b${}^*$).

We now apply Theorem~\ref{thm:finite-counter-model} to each $\Imc_e$ using $n = ||E||$ to
obtain a finite model $\Jmc_e$.
\\[1mm]
{\bf Claim.} $\Jmc = \biguplus_{e \in E^-} \Jmc_e$ satisfies the following
variations of Conditions~(a) and~(b) from Theorem~\ref{thm:charUCQs}:
\begin{enumerate}
    \item[(a$'$)] for each $e = (\Amc, q) \in E^-$: $\Jmc_e$ is a model of $\Amc$ and $\Jmc_e \not \models q$;
    \item[(b$'$)] for all $(\Amc,q) \in E^+$: if $h$ is a homomorphism from \Amc to~$\Jmc$,
        then $\Jmc_{\Amc,h, \Lmc} \models q$.
\end{enumerate}

Condition~(a$'$) follows directly from
Theorem~\ref{thm:finite-counter-model} and the choice of $n$.

For Condition~(b$'$), let $h$ be a homomorphism from $\Amc$ to $\Jmc$. As the
individual parts $\Jmc_e$ of $\Jmc$ are disjoint, consider connected components
of $\Amc$ individually. If $h$ maps a connected component of $\Amc$ to $\Jmc_e$,
then (by viewing the component as a Boolean CQ)
Theorem~\ref{thm:finite-counter-model} implies that there is a homomorphism from
this component to $\Imc_e$. By combining these homomorphisms we obtain a
homomorphism $h'$ from the entirety of $\Amc$ to $\Imc$.

Since $\Imc$ satisfies Condition~(b${}^*$), there thus must be a homomorphism
$g$ from
a CQ $p$ in $q$ to $\Imc_{\Amc, h', \Lmc}|_n$. 
Let $\Imc^*$ be the minimal connected restriction of $\Imc_{\Amc, h', \Lmc}|_n$ 
that contains the image of $g$.
Note that $\Imc^*$ is an $\Lmc$-forest model and that the depth of $\Imc^*$ is at most $n$.
We now argue that there exists a homomorphism from $\Imc^*$ to
$\Jmc_{\Amc, h, \Lmc}$ that is the identity on $\mn{ind}(\Amc)$.

For this, we first verify that the identity is a homomorphism from
$\Imc^*$ restricted to domain $\mn{ind}(\Amc)$ to $\Jmc_{\Amc, h, \Lmc}$. 
For this, consider an $a \in
\mn{ind}(\Amc)$ with $a \in A^{\Imc^*}$. 
By definition of $\Imc_{\Amc, h', \Lmc}$, $h'(a) \in A^{\Imc_e}$ for
some component $\Imc_e$ of $\Imc$. Theorem~\ref{thm:finite-counter-model}
implies that $(\Imc_e, h'(a)) \sim_{\ALCI, n} (\Jmc_e, h(a))$.
Therefore, $h(a) \in A^{\Jmc_e}$ and thus $a \in A^{\Jmc_{\Amc, h, \Lmc}}$.

Now consider a tree-shaped component $\Imc^{*\prime}$ of $\Imc^*$ that is rooted
at some $a \in
\mn{ind}(\Amc)$. Note that $|\Delta^{\Imc^{*\prime}}| \leq n$.
Thus, again using the fact that
$(\Imc_e, h'(a)) \sim_{\ALCI, n} (\Jmc_e, h(a))$, for $e$ such that $h(a) \in
\Delta^{\Jmc_e}$, we can conclude, using Lemma~\ref{lem:bisim-preserves-trees}, that
there is a homomorphism from $\Imc^{*\prime}$ to $\Jmc_e$ that maps $a$ to $h(a)$,
and therefore there is also a homomorphism from $\Imc^{*\prime}$ to $\Jmc_{\Amc, h, \Lmc}$
that maps $\Imc^{*\prime}$ to the $\Lmc$-unraveling of $\Jmc_e$ at $h(a)$ that
is attached to $a$ in $\Jmc_{\Amc, h, \Lmc}$.

Thus, there is a homomorphism from the entirety of $\Imc^*$ to $\Jmc_{\Amc, h, \Lmc}$ that
is the identity on $\mn{ind}(\Amc)$, and by composition of
homomorphisms, $\Jmc_{\Amc, h, \Lmc} \models q$, as required.

\medskip

From $\Jmc$ we now construct an ontology $\Omc$ that fits the positive and
negative examples.
The ontology $\Omc$ uses fresh concept names $V_d$ for each $d \in
\Delta^{\Jmc}$ and is constructed as follows:
\begin{align*}
    \Omc = \{ & \top \sqsubseteq \bigsqcup_{d \in \Delta^{\Jmc}} V_d \\
                & V_d \sqcap V_e \sqsubseteq \bot & & \text{for}\ d, e \in \Delta^{\Jmc} \ \ \text{with}\  d \neq e \\
                & V_d \sqsubseteq A & &\text{for}\ d \in A^{\Jmc} \\
                & V_d \sqsubseteq \neg A & &\text{for}\ d \in \Delta^{\Jmc} \setminus A^{\Jmc} \\
                & V_d \sqsubseteq \exists r. V_e & &\text{for}\ (d, e) \in r^{\Jmc} \\
                & V_d \sqsubseteq \neg \exists r.V_e & &\text{for}\ (d, e) \in (\Delta^{\Jmc} \times \Delta^{\Jmc} )\setminus r^{\Jmc} \}
\end{align*}
where $A$ ranges over concept names that occur in $E$ or have a non-empty
extension in $\Jmc$ and $r$ ranges over role names that occur in $E$ or have a
non-empty extension in
$\Jmc$ as well as over inverse roles if $\Lmc = \ALCI$.

By setting $V_d^{\Jmc} = \{d\}$ for each $d \in \Delta^{\Jmc}$, we extend $\Jmc$ to a model of $\Omc$.

Observe that in any model $\Imc$ of $\Omc$, for every $e \in \Delta^{\Imc}$,
there is exactly one $d \in \Delta^{\Jmc}$ such that $e \in V_d^{\Imc}$.
Thus, the relation
\[
 S = \{ (e, d) \in \Delta^{\Imc} \times \Delta^{\Jmc} \mid e \in V_d^{\Imc} \}
\]
is a function from left to right. 
Furthermore, using the concept inclusions in $\Omc$, one can show that $S$ is an
$\Lmc$-bisimulation and induces (due to its functionality) a homomorphism from
$\Imc$ to $\Jmc$, if $\Imc$ only interprets concept and role names that occur in $\Omc$ as non-empty. Using this property, we now show that $\Omc$ fits the examples.

First consider a negative example $e = (\Amc, q) \in E^-$. The component
$\Jmc_e$ of $\Jmc$ is a model of $\Omc$ and $\Amc$, and from Condition~(a$'$) it
follows that $\Jmc_e \not \models q$, as required.

Now consider a positive example $(\Amc, q) \in E^+$. If there are no common
models of $\Amc$ and $\Omc$, then $\Amc \cup \Omc \models q$ and we are done. 
Otherwise, let $\Imc_{\Amc}$ be a model of $\Amc$ and $\Omc$. Without loss of
generality assume that $\Imc_{\Amc}$ interprets only concept and role names that
do occur in $\Omc$ as non-empty. Note that all concept and role names used in $\Amc$ must also be
used in $\Omc$ by construction.

By the properties of $\Omc$s models, the relation $S$, as defined above, is an
$\Lmc$-bisimulation between $\Imc_{\Amc}$ and $\Jmc$, and $S$ directly induces a
homomorphism from $\Imc_{\Amc}$ to $\Jmc$.
Since $\Imc_{\Amc}$ is also a model of $\Amc$, the identity is a homomorphism
from $\Amc$ to $\Imc_{\Amc}$ and therefore by composition there is a homomorphism $h$ from
$\Amc$ to $\Jmc$.

As $\Jmc$ satisfies Condition~(b$'$), the existence of $h$ implies that
$\Jmc_{\Amc,h, \Lmc} \models q$. Let $g$ thus be a homomorphism from a CQ $p$ in $q$
to $\Jmc_{\Amc,h, \Lmc}$. 

We now aim to show that there is homomorphism $g'$ from $\Jmc_{\Amc, h, \Lmc}$ to
$\Imc_{\Amc}$ that is the identity on $\mn{ind}(\Amc)$. The composition of
$g$ and $g'$ is then a homomorphism that witnesses $\Imc_{\Amc} \models q$
as required.

For this, observe that by construction of $\Jmc_{\Amc, h, \Lmc}$, there is an
$\Lmc$-bisimulation $S'$ between $\Jmc_{\Amc, h, \Lmc}$ and $\Jmc$ with $(a, h(a)) \in S'$.
By composing $S$ and $S'$, we obtain an $\Lmc$-bisimulation $\hat S$ between
$\Jmc_{\Amc, h, \Lmc}$ and $\Imc_{\Amc}$.
Note that since $S$ agrees with $h$, $(a, a) \in \hat S$ for all $a \in
\mn{ind}(\Amc)$.

We now define $g'$ in two steps.
First set $g'(a) = a$ for all $a \in \mn{ind}(\Amc)$.
Note that by definition of $\Jmc_{\Amc, h, \Lmc}$, $g'$ is a
homomorphism from $\Jmc_{\Amc, h, \Lmc}$ restricted to $\mn{ind}(\Amc)$ to
$\Imc_{\Amc}$. 
Then, let
$a \in \mn{ind}(\Amc)$ and consider the $\Lmc$-unraveling of $\Jmc$ that was
attached to $a$ in the definition of $\Jmc_{\Amc, h, \Lmc}$. Since $S \circ S'$
is an $\Lmc$-bisimulation with $(a, a) \in S \circ S'$, we can extend $g'$ to this
entire $\Lmc$-unraveling of $\Jmc$ attached to $a$.

This completes the definition of $g'$.
\end{proof}

\section{Proofs for Section~\ref{subsect:upperboundsCQ}}

\lemIAhcomponent*
\begin{proof}
    We start with an easy observation: if $h$ is a homomorphism from $\Amc$ to
    $\Imc$, then there is a  component $\Bmc$ of $\Amc$ such that $\Imc_{\Bmc,
    h, \Lmc} \models q$. This is due to the fact that Condition~(b) from
    Theorem~\ref{thm:charUCQs} is satisfied and $q$ is connected.
    This does not yet give the lemma since the lemma requires us to find a
    component \Bmc of \Amc that works uniformly for all $h$.

    Assume to the contrary of what we have to show that for every component
    $\Bmc$ of $\Amc$ there is a homomorphism $h_\Bmc$ from $\Amc$ to $\Imc$ such
    that $\Imc_{\Bmc, h_\Bmc, \Lmc} \not \models q$. Let $h^-_\Bmc$ be
    the restriction of $h_\Bmc$ to $\mn{ind}(\Bmc)$ and let $h$ be the union of
    the homomorphisms~$h^-_\Bmc$, for all components
    \Bmc of \Amc. Then $h$ is a homomorphism from 
    $\Amc$ to $\Imc$ and $\Imc_{\Bmc, h, \Lmc} \not \models q$ for all components
    $\Bmc$ of \Amc. This, however, contradicts our initial observation.
\end{proof}

\ucqmosaic*
\begin{proof}
 We first show that if the algorithm returns `fitting exists', then there is
    an interpretation $\Imc$ that satisfies Conditions~(a) and~(b) of
    Theorem~\ref{thm:charUCQs} for $E$.

    If the algorithm returns `fitting exists', then there is a choice of $\Amf$
    and $\mn{ch}$, a base candidate $\Jmc = \bigcup_{e \in E^-} \Imc_e$ for
    $\Amf$ and $\mn{ch}$, as well as for each $e \in E^-$ a set $S_e$ of mosaics.
    We use these to construct the desired interpretation $\Imc$.

    For this, start with $\Imc = \Jmc = \bigcup_{e \in E^-} \Imc_e$,
    and for every $e \in E^-$ consider $\Imc_e$. 
    Since the algorithm returned `fitting exists` there must be, for every
    element $d$ of depth $1$ in $\Imc_e$ a mosaic $\Mmc \in S_e$ that glues to
    $d$. Then extend $\Imc_e$ by gluing $\Mmc$ to $d$, by which we mean taking
    the (non-disjoint) union of $\Imc_e$ and $\Mmc$ in which every domain element
    of $\Mmc$ is prefixed with $d$. If there are multiple candidates for $\Mmc$,
    choose arbitrarily.
    Now continue this extension process for all elements $d$ of depth $i \geq 2$
    in a breadth-first manner. Each such element $d$ must originally occur on
    depth $1$ of some mosaic $\Mmc \in S_e$. By Condition~(*) on the set $S_e$,
    there must be a $\Mmc' \in S_e$ that glues to $d$. Extend $\Imc_e$ by gluing
    $\Mmc'$ to $d$.
    
    Next, we verify that the resulting interpretation $\Imc = \bigcup_{e \in E^-} \Imc_e$
    satisfies Conditions~(a) and~(b) of Theorem~\ref{thm:charUCQs}.

    Let $e = (\Amc, p_1 \lor \cdots \lor p_k) \in E^-$.
    By construction of $\Imc$, $\Imc_e$ is an $\Lmc$-forest model of $\Amc$, and the degree of $\Imc$ is bound by \bound{E_0}.
    Assume for contradiction that $\Imc_e \models q$.
    Then, there must be a CQ $p_i$ in $q$ such that $\Imc_e \models p_i$, which
    implies that $\Imc_e \models \mn{ch}(e, p_i)$, as $\mn{ch}(e, p_i)$ is a
    connected component of $p_i$.
    By definition, $\mn{ch}(e, p_i)$ is of size at most $||E_0||$ and thus must
    match either entirely into the base candidate $\Jmc$ of $\Imc$ or entirely
    into a mosaic $\Mmc$. 
    However, both definitions demand in Condition~1 that the component $\mn{ch}(e, p_i)$ has no
    match.
    Thus, $\Imc_e \not\models q$ and Condition~(a) must hold for $\Imc$.

    Now let $e = (\Amc, q) \in E^+$ and $h$ be a homomorphism from $\Amc$ to $\Imc$.
    The existence of $h$ implies that $e$ is $\Amf$-enabled.
    As $\mn{ch}(e)$ is a component of $\Amc$, $h$ is also a homomorphism from 
    $\mn{ch}(e)$ to $\Imc$. By definition, $\mn{ch}(e)$ has size at most
    $||E_0||$.
    Thus, the range of $h$ restricted to $\mn{ch}(e)$ 
    contains either only elements of depth at most $2 ||E_0||$, or only elements of
    depth between $||E_0|| + k$ and depth $2 ||E_0|| + k$ for some $k \geq 1$.
    In the first case, the range of $h$ restricted to $\mn{ch}(e)$ lies entirely
    within the base candidate $\Jmc$ of $\Imc$.
    By Condition~4 of base candidates, thus $\Jmc_{\mn{ch}(e), h, \Lmc} \models q$ and
    $\Imc_{\mn{ch}(e), h, \Lmc} \models q$, as required.
    In the second case, the range of $h$ restricted to $\mn{ch}(e)$ lies
    entirely within depth $||E_0||$ and $2 ||E_0||$ of some mosaic $\Mmc$.
    Condition~4 of mosaics then implies that $\Mmc_{\mn{ch}(e), h, \Lmc} \models q$,
    and thus $\Imc_{\mn{ch}(e), h, \Lmc} \models q$, as required.

    \medskip
    Now we show that if there is an interpretation $\Imc$ that satisfies
    Conditions~(a) and~(b) of Theorem~\ref{thm:charUCQs}, then the algorithm
    returns `fitting exists'.

    Let $\Imc = \bigcup_{e \in E^-} \Imc_e$ be an interpretation as in
    Theorem~\ref{thm:charUCQs} for $E_0$. Most importantly, $\Imc$ has degree at most \bound{E_0}.
    By Lemma~\ref{lem:make-IAh-local}, we may assume
    that $\Imc$ also satisfies the Condition~(b${}^*$).
    By Lemma~\ref{lem:IAh-component}, for each positive example $e = (\Amc, q)
    \in E^+$, there must be a component $\Bmc$ of $\Amc$ such that if $h$ is a
    homomorphism from $\Amc$ to $\Imc$, then $\Imc_{\Bmc, h, \Lmc} \models q$. Choose
    $\mn{ch}(e)$ for all $e \in E^+$ accordingly.

    Next, consider all components $\Bmc$ of ABoxes $\Amc$ that occur in
    positive examples, and let $\Amf$ be the set of all $\Bmc$ that do not have
    homomorphisms to $\Imc$ 

    Additionally, since $\Imc$ satisfies Condition~(a), $\Imc_e \not \models q$
    for each negative example $e = (\Amc, p_1 \lor \cdots \lor p_k) \in E^-$.
    Thus, for each $p_i$, there must be a choice for $\mn{ch}(e, p_i)$ such that
    $\Imc_e \not \models \mn{ch}(e, p_i)$.

    Let $\Jmc$ now be the restriction of $\Imc$ up to depth $3 ||E_0||$. 
    We show that $\Jmc$ is a base candidate for $\mn{ch}$ and $\Amf$. 
    Condition~1 of base candidates is satisfied by choice of $\mn{ch}(e, p_i)$
    and the fact that $\Imc$ satisfies Condition~(a).
    Condition~2 of base candidates is satisfied by choice of $\Amf$.
    Condition~3 is satisfied by definition of $\Jmc$.
    For Condition~4, consider a positive example $e = (\Amc, q) \in E^+$ that is
    $\Amf$-enabled. If $h$ is a homomorphism from $\mn{ch}(e)$ to $\Imc$ whose
    range contains only element of depth at most $2 ||E_0||$, then by
    choice of $\mn{ch}(e)$, $\Imc_{\mn{ch}(e), h, \Lmc} \models q$. 
    Thus, as \Imc also satisfies Condition (b${}^*$) and the fact that the range of $h$
    contains only elements of depth at most $2 ||E_0||$, $\Jmc_{\mn{ch}(e), h, \Lmc}
    \models q$.

    Now for the mosaics. Let $e \in E^-$.
    For each $d \in \Delta^{\Imc_e}$, let $\Mmc_d$ be the restriction of
    $\Imc_e$ to all elements of the form $dw$ with $|w| < 3 ||E_0||$.

    We argue that for every for $d \in \Delta^{\Imc_e}$ with depth at
    least $1$, $\Mmc_d$ is a mosaic for $\Amf$,
    $\mn{ch}$ and $e = (\Amc, p_1 \lor \cdots \lor p_k) \in E^-$.
    Condition~1 of mosaics is satisfied by choice of $\mn{ch}(e, p_i)$ and the
    fact that $\Imc$ satisfies Condition~(a).
    Condition~2 of mosaics is satisfied by choice of $\Amf$.
    Condition~3 is satisfied by definition of $\Mmc_d$.
    Condition~4 holds by the same argument as in the base candidate case.

    Now consider the set
    \[
    S_e = \{\Mmc_d \mid d \in \Delta^{\Imc_e}, \text{depth of } d \geq 1\}.
    \]
    We claim that $S_e$ is stable, i.e., that for every $\Mmc_d \in S_e$ and
    every successor $c$ of $d$, there is an $\Mmc_d' \in S_e$ such that $\Mmc_d'$
    glues to $c$ in $\Mmc_d$. In fact this is easy to see, as $\Mmc_c \in S_e$.

    Thus, the algorithm returns `fitting exists'.
\end{proof}
\lemvariation*
\begin{proof} 
``$1 \Rightarrow 2$''. Assume that  $ \Imc_{\Amc, h, \Lmc} \models q$ and let $g$ be a strong homomorphism from a CQ $p$ in $q$ to $\Imc_{\Amc, h, \Lmc}$.
    We  use $g$ to identify an \Amc-variation $p'$ by replacing a variable $x$ with 
    $a$ if $g(x) = a\in \mn{ind}(\Amc)$ and identifying variables $x$ and $y$ if $g(x) = g(y) \notin \mn{ind}(\Amc)$. Since $ \Imc_{\Amc, h, \Lmc} \models q$ and $\Imc_{\Amc, h, \Lmc}$ is an  \Lmc-forest model 
    of \Amc, this \Amc-variation is proper. 
    Moreover, $g$ is also a homomorphism from
    $p'$ to $\Imc_{\Amc, h, \Lmc}$.
    From the construction of $\Imc_{\Amc, h, \Lmc}$ it follows that there is also a homomorphism $h'$ from $\Imc_{\Amc, h, \Lmc}$ to $\Imc$ that extends $h$.
    The composition $h' \circ g$ is then a weak homomorphism from $q$
    to \Imc. By construction, $h' \circ g$ satisfies Condition~1 of being compatible
    with $h$. Since every element in $\Imc_{\Amc,h, \Lmc}$ is $\Lmc$-reachable from some individual in $\mn{ind}(\Amc)$ in $\Imc_{\Amc, h, \Lmc}$, $h' \circ g$ also satisfies Condition~2. 
	
	\smallskip

	``$ 2 \Rightarrow 1 $''. 
    Assume that there exists a proper \Amc-variation~$p'$ of $p$ and an accompanying weak homomorphism $g$ from $p'$ to \Imc that is 
    compatible with $h$. 
   To prove that 
    $\Imc_{\Amc,h, \Lmc} \models q$, it clearly suffices to show that there is a  strong homomorphism $h'$ from $p'$ to $\Imc_{\Amc, h, \Lmc}$.
    To assemble $h'$ , we start with
    the identity mapping 
    on the set $I$ of individuals in $p'$. Since $p'$ is proper, this is a strong homomorphism from $p'$ restricted to $I$
    to $\Imc_{\Amc, h, \Lmc}$.
       
    Next, take any maximally connected component
   $q_c$ of \mbox{$p' \setminus \Amc$}. Clearly, $q_c$ contains at most one individual name and $\Imc_{q_c}$ is an $\Lmc$-tree.
   
       
        We   distinguish  two cases: 
    \begin{itemize}
        
        \item $q_c$ contains an individual, say $a$.
        Clearly $g$ is a homomorphism from $q_c$ to $\Imc$
        with $g(a) = h(a)$. As $\Imc_{q_c}$ is an $\Lmc$-tree, then by Lemma~\ref{lem:bisim-preserves-trees}, there is also
        a homomorphism from $q_c$ to the $\Lmc$-unraveling of $\Imc$ at $h(a)$
        that is attached in $\Imc_{\Amc, h, \Lmc}$ to $a$.
        Extend $h'$ accordingly.
        
        
        \item $q_c$ does not contain an individual.

        Then, let $x \in \mn{var}(q_c)$.
        As $g$ is compatible with $h$, there must be an $a \in \mn{ind}(\Amc)$ such that $g(x)$ is $\Lmc$-reachable from $h(a)$ in $\Imc$.
        Thus, as in the previous case, extend $h'$ to map $q_c$ to the $\Lmc$-unraveling of $\Imc$ that is attached to $h(a)$ in $\Imc_{\Amc, h, \Lmc}$.
        
    \end{itemize}
    Overall, we obtain the desired   strong homomorphism $h'$ from $q'$ to $\Imc_{\Amc, h, \Lmc}$.
\end{proof}

\section{Proofs for Section~\ref{sect:lowerbounds}}

\lemfirstlowercorr*
\begin{proof}
``$\Rightarrow$''
If $\Amc \cup \Omc \not \models q$, then by Lemma~\ref{lem:forest} there is an $\ALCI$-forest
 model  $\Imc$ of $\Amc$ and $\Omc$ of degree at most $||\Omc||$ such that $\Imc \not \models q$. 
By construction of $E^+$, $||\Omc|| \leq ||E^+|| \leq \bound{E}$.
 
To show that there is a ontology $\Omc'$ that fits $(E^+, E^-)$, we show that
one can extend $\Imc$ to an interpretation $\Imc'$ that satisfies
Conditions~(a) and~(b) of Theorem~\ref{thm:charUCQs} for the examples $(E^+, E^-)$.

First, set $\mn{Real}^{\Imc'} = \Delta^{\Imc}$. This ensures that $\Imc'$ is a model
of the ABox in the only negative example.
Then, for each $d \in \Delta^{\Imc}$,
introduce fresh individuals $d_1$ and $d_2$ with $(d, d_1), (d, d_2) \in s^{\Imc'}$, $d_1 \in \mn{Choice}^{\Imc'}$ and $d_2 \in \overline{\mn{Choice}}^{\Imc'}$. Additionally,
for all concept names $A$ that occur in $\Omc$ or $q$,
if $d \in A^{\Imc'}$, add $d_1$ to $A^{\Imc'}$ and $d_2$ to $\overline A^{\Imc'}$;
If $d \notin A^{\Imc'}$, add both $d$ and $d_1$ to $\overline{A}^{\Imc'}$ and $d_2$ to $A^{\Imc'}$.

Now $\Imc'$ satisfies Condition~(b) for all examples in $E^+$. In particular,
each element is labeled with exactly one of $A$, $\overline A$ for each concept name $A$ that occurs in $\Omc$ or $q$,
and there is no homomorphism from $\Amc_q$ to $\Imc'$. 
Furthermore, by construction $F^{\Imc'} = \emptyset$, and thus $\Imc'$ satisfies Condition~(a).
Therefore, by Theorem~\ref{thm:charUCQs}, there is an $\ALCI$-ontology that fits $(E^+, E^-)$.

\medskip

``$\Leftarrow$''
If there is an $\ALCI$-ontology $\Omc'$ that fits $(E^+, E^-)$, then by Theorem~\ref{thm:charUCQs}, there is an interpretation $\Imc$ that satisfies 
Conditions~(a) and~(b). We are then interested in the restriction $\Imc'$ of $\Imc$
to the domain  $\Delta^{\Imc'} = \mn{Real}^{\Imc}$.

We argue that $\Imc'$ is a model of $\Amc$ and $\Omc$ with $\Imc' \not \models q$.
First note that since $\Imc'$ satisfies Condition~(a), it is a model of $\Amc$ with
$F^{\Imc'} = \emptyset$. By construction of the positive examples that involve $\mn{Choice}$ and $\overline A$, and the fact that $\Imc$ satisfies Condition~(b), 
it follows that $(\overline A)^{\Imc'} = \Delta^{\Imc'} \setminus A^{\Imc'}$ for all concept names $A$ that occur in $\Omc$ or $q$.
Then, it follows from the positive examples that encode $\Omc$ that $\Imc'$ is a model of $\Omc$. For example, consider a concept inclusion $A \sqsubseteq \neg B \in \Omc$.
Then, there exists an example $(\{A(a)\}, \overline B(a)) \in E^+$.
Let $d \in A^{\Imc'}$. Then, $d \in \mn{Real}^\Imc$ and there is a homomorphism $h$ from $\{A(a)\}$ to $\Imc$ with $h(a) = d$, and thus $d$ must also be labeled with $\overline B(a)$. Hence, $d \notin B^{\Imc'}$ as required. The arguments for the other forms of concept inclusions are similar.

As $\Imc'$ additionally satisfies Condition~(b) for the final positive example $(\Amc_q, \exists x\, F(x))$, and $F^{\Imc'} = \emptyset$, there is no homomorphism from $\Amc_q$ to
$\Imc$, and therefore $\Imc \not \models q$.
\end{proof}


\ALCCQTwoExpHard*

We prove the theorem using a polynomial time reduction from the word problem of exponentially space-bounded alternating Turing machines (ATMs).

\begin{definition}[Alternating Turing Machine]
An \emph{alternating Turing machine} is a tuple $\Mmc = (Q, \Sigma, q_0, \Delta)$, where
\begin{itemize}
    \item $Q$ is a set of states  partitioned into a set of existential states $Q_\exists$, a set of universal states $Q_\forall$, an accepting state $\{q_{\mathrm{acc}}\}$ and 
    a rejecting state $\{q_{\mathrm{rej}}\}$,
    \item $q_0 \in Q_\exists \cup Q_\forall$ is the \emph{starting state},
    \item $\Sigma$ is a finite set of symbols, containing a blank symbol $\Box$, and
    \item $\Delta \subseteq Q \times \Sigma \times Q \times \Sigma \times \{+1 , -1\}$ is the \emph{transition relation}.
\end{itemize} 
\end{definition}
To simplify notation, we define $\Delta(q,a)= \{(q',a',M) \mid (q,a,q',a',M) \in \Delta\}$.

Let $\Mmc = (Q, \Sigma, q_0, \Delta)$ be an ATM. 
A \emph{configuration} of \Mmc is a word $wqw'$ where $w,w'\in\Sigma^*$ and $q\in Q$ and represents the content of the tape $ww'$, head position $|w|+1$, and state $q$ of the Turing machine. \emph{Successor configurations} of a configuration $wqw'$ are then defined using $\Delta$ in the standard way. 
We assume that a configuration $w q w'$ with $q \in Q_\exists \cup Q_{\forall}$ always has at least one successor configuration and that a configuration $w q w' \in \{q_{\mathrm{acc}}, q_{\mathrm{rej}}\}$ has no successor configurations.

A \emph{computation} of an ATM $\Mmc$ on a word $w$ is  sequence of configurations $K_0, K_1, \ldots$ such that $K_0 = q_0 w$ and $K_{i + 1}$ is a successor configuration of $K_i$ for all $i \geq 0$. For our purposes it suffices to consider ATMs that have finite computations on any input.
A configuration $w q w'$ without successor configurations is \emph{accepting} if $q = q_{\mathrm{acc}}$. A configuration $w q w'$ with $q \in Q_\exists$ is \emph{accepting} if at least one successor configuration is accepting. 
A configuration $w q w'$ with $q \in Q_\forall$ is \emph{accepting} if all successor configurations are accepting.
Finally an ATM $\Mmc$ \emph{accept} an input $w \in \Sigma^*$ if the configuration $q_0 w$ is accepting. If \Mmc accepts an input $w$, then this is 
witnessed by a finite computation tree whose nodes are accepting configurations,
the root being the initial configuration $q_0w$ and the leaves being configurations in state $q_{\mathrm{acc}}$ or $q_{\mathrm{rej}}$.

There is an exponentially space-bounded ATM $\Mmc = (Q, \Sigma, q_0, \Delta)$ whose word problem is \TwoExpTime-hard,
and thus, on input $w$ with $|w| = m$, all reached configurations $w' q w''$ satisfy  $|w'w''| < 2^m$~\cite{DBLP:journals/jacm/ChandraKS81}.

From $\Mmc$ and an input $w \in \Sigma^*$, we now construct a collection $E = (E^+, E^-)$ of ABox-CQ examples
such that $E$ admits a fitting $\ALC$-ontology if and only if $\Mmc$ accepts~$w$.
This is based on the characterization of  (\ALC, CQ)-ontology fitting in Theorem~\ref{thm:charUCQs}: every interpretation satisfying Conditions~(a) and~(b) of Theorem~\ref{thm:charUCQs} will represent an accepting computation of \Mmc on $w$  as a tree of configurations.

For readability, we split the construction of $E=(E^+, E^-)$ into three steps:
\begin{enumerate}
    \item We define an \ALC-ontology $\Omc$ such that 
     \ALC-forest models of $\Omc \cup \{I(a)\}$ represent the structure of accepting computation trees of $\Mmc$ on $w$, but do not yet ensure that consecutive configurations are labeled in a matching way;
    
    \item from $\Omc$, we construct ABox-CQ examples $E'$ such that every interpretation
    satisfying Conditions~(a) and~(b) of Theorem~\ref{thm:charUCQs} satisfies the same properties as models of $\Omc \cup \{I(a)\}$;
       
    \item we extend $E'$ with examples that ensure that consecutive configurations are labeled in a matching way.
\end{enumerate}

To understand the construction of $\Omc$, observe that when started in input $w$ with $|w| = m$ ,  \Mmc only visits configurations of length at most $2^m$. We can thus represent configurations of $\Mmc$ at the leaves of binary trees of height $m$.  We include in \Omc concept inclusions that generate such trees and ensure they are well-labeled.

The roots of such configuration trees are marked with the concept name~$R$.
Using the following concept inclusions, 
we enforce, at each node with label $R$, the existence of a binary tree of height $m$ and simultaneously assign each leaf its position on the tape using a binary representation based on the concept names $B_1, \ldots, B_m$, where $B_1$ represents the least significant bit:
\begin{align*}
        R & \sqsubseteq L_0\\
        L_i &\sqsubseteq \exists r. (L_{i+1} \sqcap B_{i+1}) \sqcap \exists r. (L_{i+1} \sqcap \neg B_{i+1})
\shortintertext{for $i$ with $0 \leq i < m$, and}
        L_i \sqcap B_j & \sqsubseteq \forall r. (L_{i+1} \rightarrow B_j) \\
        L_i \sqcap \neg B_j & \sqsubseteq \forall r. (L_{i+1} \rightarrow \neg B_j),
\end{align*}
for all $i, j$ with $0 < j \leq i < m$. In the following, we call the elements that are labeled with $L_m$ (tape) \emph{cells}.

In each cell, we store the current tape symbol and possibly a state $q \in \Qmc$.
To simplify the further construction, we also store the tape symbol and state of the same cell with respect to the preceding configuration.
We therefore include concept inclusions that add two elements to each cell, one labeled with the concept name $M_h$, which we will call \emph{h-memory}, and another one labeled with the concept name $M_p$, which we call \emph{p-memory},
where $h$ refers to \emph{here} and $p$ refers to \emph{preceding}. 

To be able to compare cell positions, at each h- and p-memory element we introduce 
$m$ additional successors, the $i$-th  successor being labeled with the concept name $A_i$ and storing the $i$-th bit of the position of the cell using the concept names $A$ and $\overline{A}$. Each of these successors, in turn, has a single successor to enable the construction of certain gadgets later.
For this, we add the following concept inclusions: 
\begin{align*}
        L_m &\sqsubseteq \exists r. M_h \sqcap \exists r. M_p\\
        L_m &\sqsubseteq \exists r. A_i\\
        M_h &\sqsubseteq \exists r. A_i\\
        L_m \sqcap B_i & \sqsubseteq \forall r. (A_i \rightarrow A) \sqcap \forall r^2. (A_i \rightarrow A)\\
        L_m \sqcap \neg B_i & \sqsubseteq \forall r. (A_i \rightarrow \overline A)\sqcap \forall r^2. (A_i \rightarrow \overline A)\\
        \overline A & \sqsubseteq \neg A \\
        A_i & \sqsubseteq \exists r. \top
\end{align*}
for all $i$ with $1 \leq i \leq m$.

\begin{figure}
    \centering
\resizebox{6cm}{6.5cm}{
\begin{tikzpicture}[>=Stealth, scale=1,node distance=1.0cm and 0.5cm,
  every node/.style={font=\small},
  c/.style={insert path={node[fill=black,circle, inner sep=0.6pt] {}}}]

\draw (-4,3) -- (-2.5, 5.5) -- (-1,3) -- cycle;
\draw (-1,2.7) -- (0.5, 5.2) -- (2,2.7) -- cycle;
\draw[dash dot] (-1.2,4.8) -- (-0.5,6) -- (0.2,4.8) -- cycle;
\draw[->] (-4, 6) -- (-2.5,5.5);
\draw[->] (-2.5,5.5) --( -0.5,6 );
\draw[->](-2.5,5.5) -- ( 0.5,5.2);
\draw[->] ( -0.5,6 ) -- (0.3, 6.3);
\draw[->] ( 0.5,5.2) -- (1.7,5.3);
\draw[->] ( 0.5,5.2 ) -- (1.3, 4.9);
\draw[->] ( -0.5,6) -- (0.1,5.9);
\node[fill=black, circle, inner sep=1pt, label=above:$R$] (r) at (-2.5, 5.5) {};
\node[fill=black, circle, inner sep=1pt, label=above:$R$] (r) at (0.5, 5.2) {};
\node[fill=black, circle, inner sep=0.8pt, label=above:\scriptsize ${ R}$] (r) at (-0.5, 6) {};
\node (conf1) at (-2.45,3.9) {$w'qw''$};
\node (conf2) at (0.55,3.6) {$v'q'v''$};

\draw (-1.5,3)[c] -- (-1.55, 2.7)[c] -- (-1.7, 2.4)[c] -- (-1.7,2.1)[c];
\draw (-1.55, 2.7) -- (-1.4, 2.4)[c] -- (-1.4,2.1)[c];
\draw (-1.5,3) -- (-1.3, 2.7)[c] -- (-1.3, 2.4)[c];
\draw (-1.5,3) -- (-1.1, 2.7)[c] -- (-1.1, 2.4)[c];
\draw (-1.5,3) -- (-1.8, 2.7)[c];

\draw[shift={(1.1 cm,-0.3 cm)}] (-1.5,3)[c] -- (-1.55, 2.7)[c] -- (-1.7, 2.4)[c] -- (-1.7,2.1)[c];
\draw[shift={(1.1 cm,-0.3 cm)}] (-1.55, 2.7) -- (-1.4, 2.4)[c] -- (-1.4,2.1)[c];
\draw[shift={(1.1 cm,-0.3 cm)}] (-1.5,3) -- (-1.3, 2.7)[c] -- (-1.3, 2.4)[c];
\draw[shift={(1.1 cm,-0.3 cm)}] (-1.5,3) -- (-1.1, 2.7)[c] -- (-1.1, 2.4)[c];
\draw[shift={(1.1 cm,-0.3 cm)}] (-1.5,3) -- (-1.8, 2.7)[c];

\draw[shift={(2 cm,-0.3 cm)}] (-1.5,3)[c] -- (-1.55, 2.7)[c] -- (-1.7, 2.4)[c] -- (-1.7,2.1)[c];
\draw[shift={(2 cm,-0.3 cm)}] (-1.55, 2.7) -- (-1.4, 2.4)[c] -- (-1.4,2.1)[c];
\draw[shift={(2 cm,-0.3 cm)}] (-1.5,3) -- (-1.3, 2.7)[c] -- (-1.3, 2.4)[c];
\draw[shift={(2 cm,-0.3 cm)}] (-1.5,3) -- (-1.1, 2.7)[c] -- (-1.1, 2.4)[c];
\draw[shift={(2 cm,-0.3 cm)}] (-1.5,3) -- (-1.8, 2.7)[c];

\draw[shift={(2.9 cm,-0.3 cm)}] (-1.5,3)[c] -- (-1.55, 2.7)[c] -- (-1.7, 2.4)[c] -- (-1.7,2.1)[c];
\draw[shift={(2.9 cm,-0.3 cm)}] (-1.55, 2.7) -- (-1.4, 2.4)[c] -- (-1.4,2.1)[c];
\draw[shift={(2.9 cm,-0.3 cm)}] (-1.5,3) -- (-1.3, 2.7)[c] -- (-1.3, 2.4)[c];
\draw[shift={(2.9 cm,-0.3 cm)}] (-1.5,3) -- (-1.1, 2.7)[c] -- (-1.1, 2.4)[c];
\draw[shift={(2.9 cm,-0.3 cm)}] (-1.5,3) -- (-1.8, 2.7)[c];

\draw[shift={(-1 cm,0 cm)}] (-1.5,3)[c] -- (-1.55, 2.7)[c] -- (-1.7, 2.4)[c] -- (-1.7,2.1)[c];
\draw[shift={(-1 cm,0 cm)}] (-1.55, 2.7) -- (-1.4, 2.4)[c] -- (-1.4,2.1)[c];
\draw[shift={(-1 cm,0 cm)}] (-1.5,3) -- (-1.3, 2.7)[c] -- (-1.3, 2.4)[c];
\draw[shift={(-1 cm,0 cm)}] (-1.5,3) -- (-1.1, 2.7)[c] -- (-1.1, 2.4)[c];
\draw[shift={(-1 cm,0 cm)}] (-1.5,3) -- (-1.8, 2.7)[c];

\draw[shift={(-2 cm,0 cm)}] (-1.5,3)[c] -- (-1.55, 2.7)[c] -- (-1.7, 2.4)[c] -- (-1.7,2.1)[c];
\draw[shift={(-2 cm,0 cm)}] (-1.55, 2.7) -- (-1.4, 2.4)[c] -- (-1.4,2.1)[c];
\draw[shift={(-2 cm,0 cm)}] (-1.5,3) -- (-1.3, 2.7)[c] -- (-1.3, 2.4)[c];
\draw[shift={(-2 cm,0 cm)}] (-1.5,3) -- (-1.1, 2.7)[c] -- (-1.1, 2.4)[c];
\draw[shift={(-2 cm,0 cm)}] (-1.5,3) -- (-1.8, 2.7)[c];

\draw[densely dotted] (-4.5,2.7) -- (2.6,2.7);
\draw[densely dotted] (-4.5,2.4) -- (2.6,2.4);
\draw[densely dotted] (-4.5,2.1) -- (2.6,2.1);

\node[ circle,draw ,dashed, inner sep=11pt] (d) at (-1.47, 2.5) {};

\draw[dashed] (-1.5,1.5) |- (d.south);
\draw[dashed, thick, rounded corners] (-5, -3) rectangle (3, 1.5);

\node[fill=black, circle, inner sep=2pt, label=left:$M_p$] (Ep) at (-3.5, 0) {};

\node[fill=black, circle, inner sep=2pt, label=right:$M_h$] (Eh) at (-2.5, 0) {};
\node[fill=black, circle, inner sep=2pt, label=right:$A_1$] (hA1) [below right=of Eh] {};
\node (dots1) [below =of Eh] {$\dots$};
\node[fill=black, circle, inner sep=2pt, label=left:$A_m$] (hAm) [below left=of Eh] {};
\node[fill=black, circle, inner sep=2pt] (hd1) [below=of hA1] {};
\node (dots2) [below =of dots1] {$\dots$};
\node[fill=black, circle, inner sep=2pt] (hdm) [below=of hAm] {};

\node[fill=black, circle, inner sep=2pt, label=left:$A_m$] (A1) [right= 2cm of Eh] {};
\node (dots3) [right =of A1] {$\dots$};
\node[fill=black, circle, inner sep=2pt, label=right:$A_1$] (Am) [right=of dots3] {};
\node[fill=black, circle, inner sep=2pt] (d1) [below=of A1] {};
\node (dots4) [below =of dots3] {$\dots$};
\node[fill=black, circle, inner sep=2pt] (dm) [below=of Am] {};

\node[fill=black, circle, inner sep=2pt, label=left:$L_m$] (top) at (-2,1) {};
\draw  (-2, 1.5) -- (top.north);
\draw[->] (top) -- node[above left] {} (Ep);
\draw[->] (top) -- node[above right] {} (Eh);
\draw[->] (top) -- node[above left] {} (A1);
\draw[->] (top) -- node[above right] {} (Am);
\draw[->] (Eh) -- node[above left] {} (hAm);
\draw[->] (Eh) -- node[above right] {} (hA1);
\draw[->] (hA1) -- node[above left] {} (hd1);
\draw[->] (hAm) -- node[above right] {} (hdm);
\draw[->] (A1) -- node[above left] {} (d1);
\draw[->] (Am) -- node[above right] {} (dm);
\end{tikzpicture}}

    \caption{Structure of $\ALC$-forest models of \Omc.}
    \label{fig:computationtree}
\end{figure}
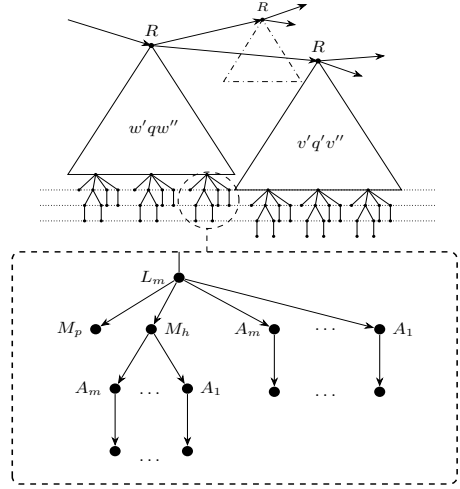

Next, we add concept inclusions which guarantee that the p-memory and h-memory of every cell are appropriately labeled
with tape symbols and states. For this, we use concept names for each $a \in \Sigma$
and $q \in Q \cup \{ \mathrm{nil} \} = Q^+$, where $\mathrm{nil}$ represents that 
the head of the ATM is currently in a different cell. For simplicity we directly use the elements of $\Sigma$ and $Q$ as concept names.
\begin{align*}
    M_h \sqcup M_p & \sqsubseteq \bigsqcup_{q \in Q^+} \bigl( q \sqcap  \bigsqcap_{q' \in Q^+ \setminus \{ q\} } \neg q' \bigr) \\
    M_h \sqcup M_p & \sqsubseteq \bigsqcup_{a \in \Sigma} \bigl( a \sqcap  \bigsqcap_{a' \in \Sigma \setminus \{a\} } \neg a' \bigr).
\end{align*}
For easier comparison, we also label each combination of symbols $a \in \Sigma$ and $q \in Q^+$ with concept name $Z_{a, q}$:
\[
M_h \sqcup M_p \sqsubseteq (a\sqcap q) \leftrightarrow Z_{a,q}
\]

We  store the current and preceding positions of the head of $\Mmc$ in the roots of configuration trees. For this we use a binary encoding via the concept names
$Q_j$ and $Q_j'$, respectively, for $1 \leq j \leq m$. 
By propagating these concept names to all cells in this configuration, we are able to ensure that only the appropriate cell is labeled with a state.
For this, we add the following concept inclusions: 
\begin{align*}
        L_i \sqcap Q_j & \sqsubseteq \forall r. (L_{i+1} \rightarrow Q_j)\\
        L_i \sqcap \neg Q_j & \sqsubseteq \forall r. (L_{i+1} \rightarrow \neg Q_j)\\
        L_i \sqcap Q_j' & \sqsubseteq \forall r. (L_{i+1} \rightarrow Q_j')\\
        L_i \sqcap \neg Q_j' & \sqsubseteq \forall r. (L_{i+1} \rightarrow \neg Q_j')\\
    L_m \sqcap C_= &\sqsubseteq \forall r. (M_h  \rightarrow \bigsqcup_{q \in Q} q) \\
    M_h \sqcap \neg C_= &\sqsubseteq \forall r. (M_h \rightarrow \mathrm{nil}),
\end{align*}
for $i, j$ with $0 \leq i < m$ and $1 \leq j \leq m$, $C_= := \bigsqcap_{i=1}^m (Q_i \leftrightarrow B_i)$, and $C'_= := \bigsqcap_{i=1}^m (Q'_i \leftrightarrow B_i)$.
Further, we need to ensure that configurations have successors that locally comply with the transition relation $\Delta$ and respect the structure of accepting ATM computation trees. Hence, we demand the existence of one successor configuration tree if the current configuration is in an existential state, and the existence of all required successor configurations  if the current configurations is in a universal state. 
For this, we store the chosen transition $(q, a, q', a', M) \in \Delta$ in 
the root of the successor configuration using a concept name $T_{a',q',M}$.
Hence, for $q \in Q_\exists$ and $a \in \Sigma$, we include
\begin{align*}
    R \sqcap \exists r^{m+1}(M_h \sqcap q \sqcap a)& \sqsubseteq \quad \bigsqcup_{ \mathclap{(q',a',M)\in \Delta(q,a)}} \  \exists r .(R \sqcap T_{q',a',M}).\\
\intertext{And for $q \in Q_\forall$ and $a \in \Sigma$, we include }
        R \sqcap \exists r^{m+1}(M_h \sqcap q \sqcap a)& \sqsubseteq \quad \bigsqcap_{\ \mathclap{(q',a',M)\in \Delta(q,a)}} \  \exists r .(R \sqcap T_{q',a',M}). 
\end{align*}

We propagate the old head position to the next configuration while changing the concept names from $Q_i$ to $Q_i'$. This enables us to distinguish the old from the new head position. For this, we include the following concept inclusions:
\begin{align*}
        Q_i \sqcap R &\sqsubseteq \forall r. (R \rightarrow Q'_i) \quad \text{ for all } 1 \leq i \leq m\\
        \neg Q_i '\sqcap R& \sqsubseteq \forall r. (R \rightarrow \neg Q'_i) \quad \text{ for all } 1 \leq i \leq m
\end{align*}

The new head position is then calculated depending on the respective transition node $T_{a',a',M}$ for $M \in \{+1, -1\}$. This is achieved by the following concept inclusions for $M = +1$:
\begin{align*}
        R \sqcap T_{q,a,+1} &\sqsubseteq Q_1' \leftrightarrow \neg Q_1 \\
        R \sqcap T_{q,a,+1} \sqcap (Q_i' \sqcap \neg Q_i) & \sqsubseteq (Q_{i+1}' \leftrightarrow \neg Q_{i+1})\\
        R \sqcap T_{q,a,+1} \sqcap (\neg Q_i' \sqcup  Q_i) & \sqsubseteq (Q_{i+1}' \leftrightarrow  Q_{i+1})\\
\intertext{for all $i$ with $1 \leq i < m$, and with the following inclusions for $M = -1$:}
        R \sqcap T_{q,a,-1} &\sqsubseteq Q_1' \leftrightarrow \neg Q_1 \\
        R \sqcap T_{q,a,-1}\sqcap (\neg Q_i' \sqcap Q_i) & \sqsubseteq (Q_{i+1}' \leftrightarrow \neg Q_{i+1})\\
        R \sqcap T_{q,a,-1} \sqcap ( Q_i' \sqcup \neg  Q_i) & \sqsubseteq (Q_{i+1}' \leftrightarrow  Q_{i+1}), 
\end{align*}
for all $1 \leq i < m$.

Using the old and new head positions as well as the transition concept $T_{a',q',M}$, we
can enforce that the cells of this configuration correctly store the result of the respective transition.
Thus, the cell at the old head position should store the symbol $a'$, the  cell at the new head position should be labeled with $q'$, and the symbol stored in all other cells should remain the same. The
latter can be encured by enforcing that for each cell, the symbol stored in the p-memory is the same as the one stored in the h-memory. We assume here that the p-memory is set up correctly, which is
not yet guaranteed but will be ensured
later on in Step~3 of the reduction.

We include the following concept inclusions for all $T_{q', a', M}$ and $a \in \Sigma$
\begin{align*}
        T_{q',a',M} & \sqsubseteq \forall r^m. (L_m \rightarrow T_{q',a',M}) \\
        L_m \sqcap T_{q',a',M} \sqcap C'_= & \sqsubseteq \forall r. (M_h \rightarrow a')\\
        L_m \sqcap T_{q',a',M} \sqcap C_= & \sqsubseteq \forall r. (M_h \rightarrow q')\\
        L_m \sqcap \exists r. (M_p \sqcap a \sqcap \mathrm{nil}) &\sqsubseteq \forall r. (M_h \rightarrow a),
\end{align*}
where $C'_= := \bigsqcap_{i=1}^m (Q'_i \leftrightarrow B_i)$.

To ensure that the represented run is accepting, we demand that every configuration without successor states must be in the accepting state:
\[
R \sqcap\forall r. \neg R \sqsubseteq \bigsqcup_{a \in \Sigma, M \in \{-1, +1\} }T_{a,q_{\mathrm{acc}},M}
\]

The final concept inclusions are used to encode the initial configuration. For this, we will take care that the root of the computation tree is labeled with the
concept name $I$. We then put:
\begin{align*}
        I & \sqsubseteq \exists r. R\\
        I &\sqsubseteq \forall r^{m+1}. (\mn{pos}_B= i \rightarrow \forall r.(M_h \rightarrow a_i))  \\
        \intertext{for all $i$ with $0 \leq i < m$, and}
        I &\sqsubseteq \forall r^{m+1}. (\mn{pos}_B= 0 \rightarrow \forall r.(M_h \rightarrow q_0))\\
        I &\sqsubseteq \forall r^{m+1}. (\mn{pos}_B\geq m \rightarrow \forall r.(M_h \rightarrow \Box))\\
        I &\sqsubseteq \forall r. \neg Q_i,\ \text{for all $i$ with $1 \leq i \leq m$,}
\end{align*}
where $\mn{pos}_B = 0$, $\mn{pos}_B = i$ and $\mn{pos}_B \geq m$ are abbreviations for Boolean combinations of the concept
names $B_i$ which make sure that the position is as required.

This completes the construction of $\Omc$. The models of $\Omc$ and the ABox $\{I(a)\}$
are now almost accepting runs of $\Mmc$ on $w$. What is missing is the following property:
\begin{center}
($*$) If two cells $t$ and $t'$ of succeeding configurations have the same position, then
the h-memory of $t$ must have the same label $Z_{a, q}$ as the p-memory of $t'$.
\end{center}
We next convert the constructed ontology $\Omc$ to a collection of labeled ABox-CQ examples, and then add certain further ABox-CQ examples that
enforce ($*$).

\smallskip\noindent
\textbf{Claim~1.} There is a collection $E'$ of ABox-CQ examples of size polynomial in $||\Omc||$ such that
\begin{enumerate}
    \item every $\ALC$-forest model of $\Omc$ and $\{I(a)\}$ can be extended to a model that satisfies Conditions~(a) and~(b) of Theorem~\ref{thm:charUCQs} for $E'$;
    \item every model that satisfies Conditions~(a) and~(b) of Theorem~\ref{thm:charUCQs} for $E'$ is a model of $\Omc$ and $\{I(a)\}$.
\end{enumerate}

Claim~1 in fact can be shown by using exactly the same examples as in the proof of Theorem~\ref{thm:alcicq-2exphard}.

\medskip

Now we extend $E'$ with additional ABox-CQ examples that enforce ($*$) to obtain the final collection of examples $E$.
%
To get a first intuition for the following  construction, we ask the reader to recall the example of Theorem~\ref{thm:alcicq-2exphard} in Figure~\ref{fig:cruicalexample} together with the fact that any strong homomorphism from $q^*$ to 
$\Imc_{\Amc^*, h, \Lmc}$ maps $x$ to $c$ or $d$ for every model \Imc. Intuitively, this gadget uses the idea that we may force existential variables to act like answer variables in the sense that they
can only bind to the individuals in $\ALC$-forest models, but not to elements in the trees of those models. 
It is not hard to see that 
all variables in a CQ $q$ that have a directed path to an individual in $q$ are forced to act like such an answer variable. Note that this statement is only true for \ALC, where the edges in the trees
of a forest model must be directed away from the ABox.

Building upon these fundamental ideas, we now construct the examples that ensure the contraposition of ($*$). We use an ABox that matches whenever the $Z_{a,q}$ labels of two cells of consecutive configurations do not agree and then use the corresponding query to ensure that their respective positions do not coincide. The ABox consists of multiple components, only connected through a few designated individuals and each of these components matches into a different bit of the positional memory. 
The components and the query are carefully crafted in such way that all variables act like answer variables and must collectively map into exactly one of these components. Then, using the concepts $A$ and $\overline A$, this choice presents a certificate for the disagreement of both positions. 

For all concept names $Z_{a,q}, Z_{a',q'}$ with $Z_{a, q} \neq Z_{a', q'}$, we define the following positive example $(\Amc^*, q^*)$ with $\Amc^* = \Amc_1 \cup \ldots \cup \Amc_m$ and
\begin{align*}
\Amc_i =\{&&&  r(c^i_0, c^i_1), \ldots , r(c^i_{m-1}, c^i_{m}), r(c^i_{m}, c), \\
    &&& r(\hat{c}^i_0, \hat{c}^i_1), \ldots , r(\hat{c}^i_{m-1}, \hat{c}^i_{m}), r(\hat{c}^i_{m}, \hat{c}),\\
    &&& r(c_0^i,a_0^i), r(a_0^i, a_1^i), \ldots, r(a_{m-1}^i, a_{m}^i), r(a_{m}^i, c^i),\\
    &&& r(\hat{c}_0^i,{b}_0^i), r({b}_0^i, {b}_1^i), \ldots, r({b}_{m-1}^i, {b}_{m}^i), r({b}_{m}^i, c^i),\\
    &&& r(\hat{c}_0^i,\hat{a}_0^i), r(\hat{a}_0^i, \hat{a}_1^i), \ldots, r(\hat{a}_{m-1}^i, \hat{a}_{m}^i), r(\hat{a}_{m}^i, \hat{c}^i), \\
    &&& r(c_0^i,\hat{b}_0^i), r(\hat{b}_0^i, \hat{b}_1^i), \ldots, r(\hat{b}_{m-1}^i, \hat{b}_{m}^i), r(\hat{b}_{m}^i, \hat{c}^i),\\
    &&& A_i({c}^i),r(c, d_1^i), r(d_1^i, d_2^i) ,r(c^i, d_2^i),\\
    &&& A_i(\hat{c}^i),r(\hat{c}, \hat{d}_1^i), r(\hat{d}_1^i, \hat{d}_2^i),r(\hat{c}^i, \hat{d}_2^i), \\
    &&& R(c_0^i), R(\hat{c}_1^i),\\
    &&& r(\hat{c}, \hat{p}),  Z_{a',q'}(\hat{p}), M_p(\hat{p}),\\
    &&&  Z_{a,q}(c), M_h(c)\}  \qquad (\text{see Figure~\ref{fig:aboxred}})
\end{align*}

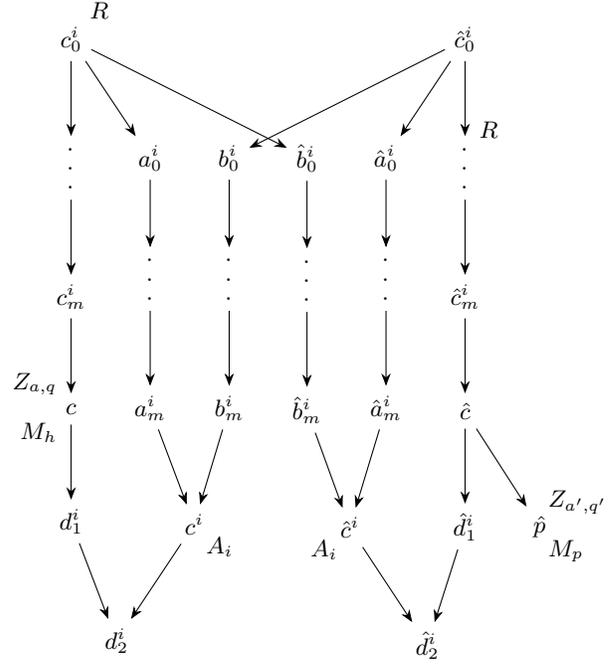
\begin{figure}
    \centering
    \hspace*{-0.5cm}
     \begin{tikzpicture}[
  node distance=1.0cm and 0.5cm,
  every node/.style={font=\small,  minimum size=0.4cm},
  ->, >=Stealth
]
\node (ci0) at (0, 0) {$c^i_0$};
\node (R) [above right=-0.2cm of ci0] {$R$};
\node (ai1) [below right=of ci0] {$a^i_0$};
\node (bhi1) [right =of ai1] {${b}^i_0$};
\node (bi1) [right =of bhi1] {$\hat{b}^i_0$};
\node (ahi1) [right =of bi1] {$\hat{a}^i_0$};
\node (chi0) [above right=of ahi1] {$\hat{c}^i_0$};

\node (cidots) [below=of ci0] {\rotatebox{90}{$\cdot\;\cdot\;\cdot$}};
\node (chidots) [below=of chi0] {\rotatebox{90}{$\cdot\;\cdot\;\cdot$}};
\node (R2) [above right=-0.2cm of chidots] {$R$};
\node (aidots) [below=0.9cm of ai1] {\rotatebox{90}{$\cdot\;\cdot\;\cdot$}};
\node (bhidots) [below=0.9cm of bhi1] {\rotatebox{90}{$\cdot\;\cdot\;\cdot$}};
\node (bidots) [below=0.9cm of bi1] {\rotatebox{90}{$\cdot\;\cdot\;\cdot$}};
\node (ahidots) [below=0.9cm of ahi1] {\rotatebox{90}{$\cdot\;\cdot\;\cdot$}};

\node (cim) [below=of cidots] {${c}^i_m$};
\node (chim) [below=of chidots] {$\hat{c}^i_m$};
\node (aim) [below=of aidots] {${a}^i_m$};
\node (ahim) [below=of ahidots] {$\hat{a}^i_m$};
\node (bim) [below=of bidots] {$\hat{b}^i_m$};
\node (bhim) [below=of bhidots] {${b}^i_m$};

\node (c) [below=of cim] {${c}$};
\node (ch) [below=of chim] {$\hat{c}$};

\node (ci) [below right=1cm and 0.0cm of aim] {${c}^i$};
\node (chi) [below right=1cm and 0.0cm of bim] {$\hat{c}^i$};

\node (AI) [below right=-0.3cm of ci] {$A_i$};

\node (AHI) [below left=-0.3cm of chi] {$A_i$};

\node (d1i) [below =of c] {$d_1^i$};

\node (d1hi) [below =of ch] {$\hat{d}_1^i$};

\node (d2i) [below left=of ci] {$d_2^i$};

\node (d2hi) [below right=of chi] {$\hat{d}_2^i$};

\node (p) [right=of d1hi] {$\hat{p}$};

\node (L) [above left=-0.2cm of c] {$Z_{a,q}$};
\node (Lh) [above right=-0.3cm of p] {$Z_{a',q'}$};

\node (M) [below left=-0.2cm of c] {$M_h$};
\node (Mh) [below right=-0.3cm of p] {$M_p$};

\draw (ch) -- (p);
\draw (ci0) -- (cidots);
\draw (chi0) -- (chidots);
\draw (ai1) -- (aidots);
\draw (ahi1) -- (ahidots);
\draw (bi1) -- (bidots);
\draw (bhi1) -- (bhidots);
\draw (ci0) -- (ai1);
\draw (ci0) -- (bi1);
\draw (chi0) -- (ahi1);
\draw (chi0) -- (bhi1);
\draw (cidots) -- (cim);
\draw (chidots) -- (chim);
\draw (aidots) -- (aim);
\draw (bidots) -- (bim);
\draw (ahidots) -- (ahim);
\draw (bhidots) -- (bhim);
\draw (cim) -- (c);
\draw (chim) -- (ch);
\draw (aim) -- (ci);
\draw (ahim) -- (chi);
\draw (bim) -- (chi);
\draw (bhim) -- (ci);
\draw (c) -- (d1i);
\draw (ci) -- (d2i);
\draw (ch) -- (d1hi);
\draw (chi) -- (d2hi);
\draw (d1i) -- (d2i);
\draw (d1hi) -- (d2hi);

\end{tikzpicture}
    \caption{The ABox $\Amc_i \subseteq \Amc^*$, where arrows depict $r$-roles.}
    \label{fig:aboxred}
\end{figure}
\begin{align*}
q^* = \{&&& r(z_0, z_1), \ldots, r(z_{m-1}, z_{m}), r(z_{m}, c),\\
    &&& r(\hat{z}_0, \hat{z}_1), \ldots, r(\hat{z}_{m-1}, \hat{z}_{m}), r(\hat{z}_{m}, \hat{c}),\\
    &&& r(z_0, x_0), r(x_0, x_1), \ldots, r(x_{m-1}, x_{m}), r(x_{m}, z),\\
    &&& r(\hat{z}_0, \hat{y}_0), r(\hat{y}_0, \hat{y}_1), \ldots, r(\hat{y}_{m-1}, \hat{y}_{m}), r(\hat{y}_{m}, z),\\
    &&& r(\hat{z}_0, \hat{x}_0), r(\hat{x}_0, \hat{x}_1), \ldots, r(\hat{x}_{m-1}, \hat{x}_{m}), r(\hat{x}_{m}, \hat{z}),\\
    &&& r(z_0, y_0), r(y_0, y_1), \ldots, r(y_{m-1}, y_{m}), r(y_{m}, \hat{z}),\\
    &&& A(z), \overline A (\hat{z})\} \qquad (\text{ see Figure~\ref{fig:queryred}})
\end{align*}

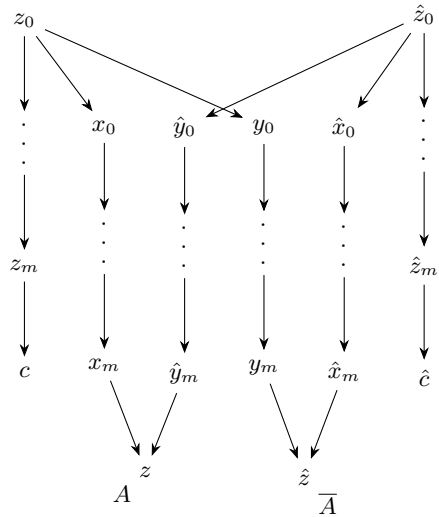
\begin{figure}
    \centering
    \hspace*{-0.3cm}
     \begin{tikzpicture}[
  node distance=1.0cm and 0.5cm,
  every node/.style={font=\small,  minimum size=0.4cm},
  ->, >=Stealth
]
\node (ci0) at (0, 0) {$z_0$};
\node (ai1) [below right=of ci0] {$x_0$};
\node (bhi1) [right =of ai1] {$\hat{y}_0$};
\node (bi1) [right =of bhi1] {${y}_0$};
\node (ahi1) [right =of bi1] {$\hat{x}_0$};
\node (chi0) [above right=of ahi1] {$\hat{z}_0$};

\node (cidots) [below=of ci0] {\rotatebox{90}{$\cdot\;\cdot\;\cdot$}};
\node (chidots) [below=of chi0] {\rotatebox{90}{$\cdot\;\cdot\;\cdot$}};

\node (aidots) [below=0.9cm of ai1] {\rotatebox{90}{$\cdot\;\cdot\;\cdot$}};
\node (bhidots) [below=0.9cm of bhi1] {\rotatebox{90}{$\cdot\;\cdot\;\cdot$}};
\node (bidots) [below=0.9cm of bi1] {\rotatebox{90}{$\cdot\;\cdot\;\cdot$}};
\node (ahidots) [below=0.9cm of ahi1] {\rotatebox{90}{$\cdot\;\cdot\;\cdot$}};

\node (cim) [below=of cidots] {${z}_m$};
\node (chim) [below=of chidots] {$\hat{z}_m$};
\node (aim) [below=of aidots] {${x}_m$};
\node (ahim) [below=of ahidots] {$\hat{x}_m$};
\node (bim) [below=of bidots] {${y}_m$};
\node (bhim) [below=of bhidots] {$\hat{y}_m$};

\node (c) [below=of cim] {${c}$};
\node (ch) [below=of chim] {$\hat{c}$};

\node (ci) [below right=1cm and 0.0cm of aim] {${z}$};
\node (chi) [below right=1cm and 0.0cm of bim] {$\hat{z}$};

\node (AI) [below left=-0.2cm of ci] {$A$};

\node (AHI) [below right=-0.2cm of chi] {$\overline A$};

\draw (ci0) -- (cidots);
\draw (chi0) -- (chidots);
\draw (ai1) -- (aidots);
\draw (ahi1) -- (ahidots);
\draw (bi1) -- (bidots);
\draw (bhi1) -- (bhidots);
\draw (ci0) -- (ai1);
\draw (ci0) -- (bi1);
\draw (chi0) -- (ahi1);
\draw (chi0) -- (bhi1);
\draw (cidots) -- (cim);
\draw (chidots) -- (chim);
\draw (aidots) -- (aim);
\draw (bidots) -- (bim);
\draw (ahidots) -- (ahim);
\draw (bhidots) -- (bhim);
\draw (cim) -- (c);
\draw (chim) -- (ch);
\draw (aim) -- (ci);
\draw (ahim) -- (chi);
\draw (bim) -- (chi);
\draw (bhim) -- (ci);

\end{tikzpicture}
    \caption{The query $q^*$, where arrows  depict $r$-roles.}
    \label{fig:queryred}
\end{figure}

Observe that the two ABoxes $\Amc_i, \Amc_j$ with $i \neq j$ have only three individuals in common, namely $c$, $\hat{c}$ and $\hat{p}$. Adding each of these positive examples to $E'$, we obtain $E$.

We are now able to state the intended behaviour of $(\Amc^*,q^*)$ in a precise way.\\
\smallskip \noindent
\textbf{Claim~2.} Let $\Imc$ be an $\ALC$-forest model of $\Amc^*$. If $h$ is a (strong) homomorphism from $q^*$ to
$\Imc$, then 
\begin{enumerate}
    \item $h(x) \in \mn{ind}(\Amc^*)$ for all $x \in \mn{var}(q^*)$,
    \item $\{h(z), h(\hat z)\}  \subseteq \{ c^i, \hat c^i \}$ for some $i$ with $1 \leq i \leq m$.
\end{enumerate}

\smallskip \noindent
\textbf{Proof of Claim~2.}
We show both points at the same time.
By definition of strong homomorphisms,
$h(c) = c \in \mn{ind}(\Amc^*)$ and $h(\hat c) = \hat c \in \mn{ind}(\Amc^*)$.
As $\ALC$-forest models do not contain predecessors of individuals in $\Amc^*$ that do not occur in $\Amc^*$,
$h(z_k) = c^i_k$ and $h(\hat z_k) = \hat c^j_k$ for all $k$ with $0 \leq k \leq m$
and for some $i, j$ with $1 \leq i, j \leq m$.
Observe now that $q^*$ contains a directed $r$-path of length $m + 2$ from $z_0$ to $z$
and a directed $r$-path of length $m + 2$ from $\hat z_0$ to $z$.
Again due to the structure of $\ALC$-forest models of $\Amc^*$, $c^i_0$ and $\hat c^j_0$ only have common $r$-successors
at depth $m + 2$ for $i = j$, and these are $c^i$ and $\hat c^i$. Thus, $i = j$, and $h(x_k) = a^i_k$ and $h(\hat y_k) = b^i_k$, or $h(x_k) = \hat b^i_0$ and $h(\hat y_k) = \hat a^i_k$, for all $k$ with $0 \leq k \leq m$, and therefore $h(z) \in \{c^i, \hat c^i\}$.
Thus the same argument applies to $\hat z$ and the variables $y_k$ and $\hat x_k$.
This completes the proof of the claim.

\smallskip


Now consider the intended $\ALC$-forest model $\Imc$ of $\Omc$ depicted in Figure~\ref{fig:computationtree}, and a homomorphism $h$ from $\Amc^*$ to $\Imc$. By the tree-shape of $\Imc$, 
$h(c^i_0) = h(\hat c^i_0)$ and $h(c^i_0) = h(c^j_0)$ for all $i, j$ with $1 \leq i, j \leq m$,
and $h(c^i_0)$ must be the root (labeled with $R$) of a configuration $K$.
From the lengths of $r$-paths,
it follows that $h(c)$ is a h-memory node of some cell in this configuration,
and from the paths involving $d_1^i$ and $d^i_2$, it follows that $h(c) = h(a^i_m) = h(b^i_m)$,
and $h(c^i)$ is the element that stores the $i$-th bit of the cell position.

Again the lengths of $r$-paths imply that $h(\hat p)$ is the p-memory node of
some cell in a \emph{successor} configuration $K'$ of $K$. And by the same argument above it follows that
$h(\hat c^i)$ is the element that stores the $i$-th bit of this cell position.

Together with Claim~2, these observations suffice to show that this reduction is correct:
\begin{enumerate}
    \item \Mmc accepts $w$
    \item There exists an $\ALC$-ontology that fits $E$ 
\end{enumerate}
\begin{proof}
``$1 \Rightarrow 2$''
Assume there is an accepting computation of $\Mmc$ on $w$, and
let $\Imc'$ be the corresponding model of $\Omc$ and $\{I(a)\}$.
By construction of the examples $E'$, $\Imc'$ can then, as in the proof of Theorem~\ref{thm:alcicq-2exphard}, be extended to an interpretation
$\Imc$ that satisfies Conditions~(a) and~(b) of Theorem~\ref{thm:charUCQs} for the examples $E'$.
Now we verify that $\Imc$ also satisfies Condition~(b) for the positive examples we added to construct $E$.
Let $(\Amc^*, q^*)$ be a positive example as constructed above for some $Z_{a, q}, Z_{a', q'}$ with $Z_{a, q} \neq Z_{a', q'}$, and let $h$ be a homomorphism from $\Amc^*$ to $\Imc$. By the above observations about $\Amc^*$, it follows that
$h(c)$ is a cell labeled with $Z_{a, q}$, and for all $i$, $h(c^i)$ is the element that stores the $i$-th bit of the position of this cell. Furthermore, $h(\hat c)$ is a cell in a successor configuration, where the p-memory is labeled with $Z_{a', q'}$
and for all $i$, $h(\hat c^i)$ is the element that stores the $i$-th bit of the position of this cell.

As $\Imc$ represents a well-formed computation tree, the p-memory of a cell is always labeled with the same $Z_{a, q}$ as the $h$-memory of the same cell in the predecessor configuration as in ($*$),
it follows that $h(c)$ and $h(\hat c)$ must be in cells with different cell positions. Thus,
there is an $i$ such that $h(c^i) \in A^{\Imc}$ and $h(\hat c^i) \in \overline A^{\Imc}$, or vice versa.
This in turn implies $c^i \in A^{\Imc_{\Amc^*, h, \ALC}}$ and $\hat c^i \in \overline A^{\Imc_{\Amc^*, h, \ALC}}$, or vice versa.
Thus, we can directly construct a (strong) homomorphism $g$ from $q^*$ to $\Imc_{\Amc^*, h, \ALC}$ with $g(z) = c^i$ and $g(\hat z) = \hat c^i$, or vice versa, witnessing $\Imc_{\Amc^*, h, \ALC} \models q^*$, as required.

\smallskip
``$2 \Rightarrow 1$''
By Theorem~\ref{thm:charUCQs} there is an $\ALC$-forest model $\Imc$ that satisfies Conditions~(a) and~(b) for the examples $E$.
Since $\Imc$ also satisfies Conditions~(a) and~(b) for the examples $E'$, $\Imc$ is a model of $\Omc$ and $\{I(a)\}$.
Thus, $\Imc$ represents an accepting computation tree of $\Mmc$ on $w$ if it satisfies ($*$).
We will show the contraposition of ($*$).
Thus, let $t$ and $t'$ be cells of succeeding configurations. If the h-memory of $t$ is labeled with $Z_{a,q}$ 
and the p-memory of $t'$ is labeled with $Z_{a', q'}$ and $Z_{a, q} \neq Z_{a', q'}$, then
there is a positive example $(\Amc^*, q^*)$ constructed for $Z_{a, q}$ and $Z_{a', q'}$ in $E$ and
there is a homomorphism $h$ from $\Amc^*$ to $\Imc$ with $h(c)$ lies in $t$ and $h(\hat c)$ lies in $t'$.
As $\Imc$ satisfies Condition~(b) for the example $(\Amc^*, q^*)$, it follows that $\Imc_{\Amc^*, h, \ALC} \models q^*$.
Thus, there is a (strong) homomorphism $g$ from $q^*$ to $\Imc_{\Amc^*, h, \ALC}$. As $\Imc_{\Amc^*, h, \ALC}$ is an $\ALC$-forest model
of $\Amc^*$, $g(z), g(\hat z) \subseteq \{ c^i, \hat c^i\}$ for some $i$ by Claim~2.
Since $ \overline A \sqsubseteq \neg A \in \Omc$, and $\Imc$ is a model of $\Omc$, $g(z) = c^i$ and $g(\hat z) ) = \hat c^i$ 
or vice versa. Thus $c^i \in A^{\Imc_{\Amc^*, h, \ALC}}$ and $\hat c^i \in \overline A^{\Imc_{\Amc^*, h, \ALC}}$, or vice versa.
In both cases, the position of $t$ and the position of $t'$ must differ at bit $i$. Hence, $\Imc$ satisfies ($*$) and represents a accepting computation tree of $\Mmc$ on~$w$.
\end{proof}

\section{Proofs for Section~\ref{sect:concl}}

Let $\Imc_1$ and $\Imc_2$ be interpretations. We write $(\Imc_1,d_1) \sim_\ALCQ (\Imc_2,d_2)$ for $d_1 \in \Delta^{\Imc_1}$ and $d_2 \in \Delta^{\Imc_2}$ if there exists a \emph{counting bisimulation} $R \subseteq \Delta^{\Imc_1} \times \Delta^{\Imc_2}$ with  $(d_1,d_2) \in R$. For formal definitions and the following proposition, we refer to \cite{DBLP:conf/ijcai/LutzPW11}.

\begin{proposition}\label{prop:alcqbisim}
    Let $\Imc_1$ and $\Imc_2$ be interpretations, and \Omc an \ALCQ-ontology.
    If $\Imc_2 \models \Omc$ and for each $d_1\in \Delta^{\Imc_1} $ there exists $d_2 \in \Delta^{\Imc_2} $ such that $(\Imc_1,d_1) \sim_\ALCQ (\Imc_2,d_2)$, then $\Imc_1$ is a model of \Omc.
\end{proposition}

In the context of \ALCQ, the \emph{unraveling} $\Imc_d$ of an interpretation $\Imc$ at $d\in \Delta^\Imc$ is defined analogously to the \ALC case. It is well known that for each $d'\in \Delta^{\Imc_d} $ there exists $e \in \Delta^{\Imc} $ such that $(\Imc_d,d') \sim_\ALCQ (\Imc,e)$. 

We say that a homomorphism $h$ from an ABox $\Amc_1$ to an ABox $\Amc_2$ is
\emph{locally injective} if $h(b) \neq h(c)$ for all $r(a,b),r(a,c) \in \Amc_1$ \cite{DBLP:conf/ijcai/FunkJLPW19}.

\begin{proposition}
\label{cor:modfrominjhom}
	Let \Imc be an ABox, \Omc an \ALCQ-ontology and \Imc an interpretation 
	with $ \Imc \models \Omc $. 
    If there is a locally injective homomorphism from \Amc to \Imc, then there exists
	a model \Jmc of $\Omc \cup \Amc$.
\end{proposition}
\begin{proof}
Let $h$ be a locally injective homomorphism from \Amc to \Imc.
$\Jmc$ is then constructed the following way:

We begin with adding $\mn{ind}(\Amc)$ to $ \Delta^{\mathcal{J}}$
and, for all $a \in \mn{ind}(\Amc)$ and $A \in \NC$, adding  $a$ to $A^{\mathcal{J}} $ if $ h(a) \in A^{\mathcal{I}}$. Furthermore, for each $a \in \mn{ind}(\Amc)$, $r \in \NR$ and $ d \in \Delta^{\Imc} \setminus 
\{ h(b) \mid r(a, b) \in \Amc \}$ with 
$ \left( h(a), d \right) \in r^\Imc $, add a copy of $ \Imc_d $ to 
\Jmc and connect $ a $ to it's root via $ r $. It is easy to see that 
combining the canonical counting bisimulations of the copies of $ \mathcal{I}_d $ 
with $ h $ results in a valid counting bisimulation between  $ \mathcal{J} $ and $ \mathcal{I} $, and thus
for each $e\in \Delta^{\Jmc} $ there exists $e' \in \Delta^{\Imc} $ such that $(\Jmc,e) \sim_\ALCQ (\Imc,e')$.
An application Proposition~\ref{prop:alcqbisim} shows that \Jmc is a model of \Omc. 
\end{proof}



\thmALCQ*
\begin{proof}
	``$ 1 \Rightarrow 2 $'' Assume there exists $ \Amc \in E^{-} $ 
	and a locally injective homomorphism $ h $ from $\Amc$ to  $\mathcal{A}^{+}$.
	Now suppose, contrary to what we want to show, that there is a fitting \ALCQ-ontology 
	\Omc. Since \Omc is consistent with every $\Amc \in E^{+} $ 
	it is consistent with $ \mathcal{A}^{+} $, witnessed by the disjoint union $ \mathcal{I} $ 
	of the individual
	models. An application of Proposition~\ref{cor:modfrominjhom} 
    on $\Amc$, \Imc and the composition $g \circ h$, where
    $g$ is the homomorphism from
    $\Amc^+$ to \Imc, proves that $\Amc$
    is consistent with \Omc. This contradicts the 
    assumption that \Omc is a fitting ontology for $E$.

	\smallskip

	``$ 2 \Rightarrow 1 $''
	Assume there is no locally injective homomorphism from any $\Amc \in E^-$ to $\Amc^+$. We use $ \Amc^{+} $ to construct the fitting
	ontology \[ \mathcal{O}:= \Omc_{\Amc^+,\Sigma} \cup 
    \bigcup_{a \in \mn{ind}(\Amc^+)}\{\top 
    \sqsubseteq (\leq 1.r) V_a \mid r \in \Sigma\} ,\] where $\Omc_{\Amc^+,\Sigma}$ is the \ALC-ontology for $\Amc^+$ defined in Proposition~\ref{prop:fromKR12}, for $\Sigma = \mn{sig}(E)$. 
    
    It is straightforward to show that each $\Amc \in E^+$ 
    is consistent with \Omc by utilizing $\Imc_\Amc$, where we set $V_a^{\Imc_\Amc}:=\{a\}$ for each $\mn{ind}(\Amc)$. 
    Thus, it remains to prove the inconsistency with every $\Amc \in E^-$.
    Contrary to what we want to show, assume that some $\Amc \in E^-$ is 
    consistent with \Omc and let $\Imc$ be a model of $\Amc$ and $\Omc$. Combining \Imc with Proposition~\ref{prop:fromKR12}, we receive a homomorphism $ h$ from \Amc to $\Amc^{+} $, where $ a \in V_b ^{\mathcal{I}} $ if $ h(a) = a_b $.
    
	We now claim that $ h $ is locally injective. Suppose $r(a,b), r(a,c) \in \Amc$, then there exists no 
    $V_d$ with $b,c \in V_d^\Imc$, by definition of \Omc. Lifting this observation back to $h$ shows $h(b) \neq h(c)$, and thus the local injectivity. As the existence of such a homomorphism contradicts the initial assumption, we conclude that \Omc is a fitting ontology for $E$.
\end{proof}

\end{document}